\documentclass[journal]{IEEEtranTIE}

\usepackage{graphicx}
\usepackage{cite}
\usepackage{picinpar}
\usepackage{url}
\usepackage{flushend}
\usepackage{colortbl}
\usepackage{soul}
\usepackage{multirow}
\usepackage{pifont}
\usepackage{color}
\usepackage{alltt}
\usepackage[hidelinks]{hyperref}
\usepackage{enumerate}
\usepackage{siunitx}
\usepackage{breakurl}
\usepackage{epstopdf}
\usepackage{pbox}
\usepackage[numbers,sort&compress]{natbib}
\usepackage{amsthm,amsmath,amssymb}
\usepackage{mathrsfs}
\usepackage{mathtools}
\usepackage[caption=false]{subfig}
\usepackage{xcolor}
\usepackage{pgfplots} 
\usepackage{placeins} 
\usepackage{algorithm}
\usepackage{algpseudocode} 
\usepackage{tabularx}  
\usepackage{booktabs}  
\usepackage{float}     

\usepackage[utf8]{inputenc} 
\usepackage[T1]{fontenc}    
\usepackage[noprefix]{nomencl} 
\usepackage{etoolbox}      

\pgfplotsset{compat=newest} 

\makenomenclature
\renewcommand{\nomname}{List of Symbols}
\patchcmd{\thenomenclature}
{\chapter*{\nomname}}
{\section*{\nomname}}
{}{} 

\renewcommand{\figurename}{Fig.~}

\newcommand{\R}{\mathbb{R}}

\newtheorem{lemma}{Lemma}

\newtheorem{theorem}{Theorem}

\newtheorem{assumption}{Assumption}

\begin{document}

\title{Meanshift Shape Formation Control Using\\ Discrete Mass Distribution}

\author{Yichen~Cai, Yuan Gao, Pengpeng Li, Wei Wang, \emph{Senior Member, IEEE}, \\Guibin~Sun, and~Jinhu~L\"u, \emph{Fellow, IEEE}

\thanks{This work was supported in part by the National Key Laboratory of Multi-perch Vehicle Driving Systems under Grant QDXT-NZ-202407-01. \emph{(Corresponding author: Guibin~Sun.)}}
\thanks{Yichen~Cai and Pengpeng Li are with the School of Artificial Intelligence, Beihang University, Beijing 100191, China (e-mail: caiyichen@buaa.edu.cn; pengpengli@buaa.edu.cn).}
\thanks{Yuan~Gao, Wei~Wang, Guibin~Sun and Jinhu~L\"u are with School of Automation Science and Electrical Engineering, Beihang University, Beijing 100191, China (e-mail: gy\_sasee@buaa.edu.cn; w.wang@buaa.edu.cn; sunguibinx@buaa.edu.cn; jhlu@iss.ac.cn).}
}

\maketitle

\begin{abstract}
	The density-distribution method has recently become a promising paradigm owing to its adaptability to variations in swarm size.
	However, existing studies face practical challenges in achieving complex shape representation and decentralized implementation. 
	This motivates us to develop a fully decentralized, distribution-based control strategy with the dual capability of forming complex shapes and adapting to swarm-size variations.	
	Specifically, we first propose a discrete mass-distribution function defined over a set of sample points to model swarm formation.
	In contrast to the continuous density-distribution method, our model eliminates the requirement for defining continuous density functions---a task that is difficult for complex shapes. 
	Second, we design a decentralized meanshift control law to coordinate the swarm's global distribution to fit the sample-point distribution by feeding back mass estimates. 
	The mass estimates for all sample points are achieved by the robots in a decentralized manner via the designed mass estimator.
	It is shown that the mass estimates of the sample points can asymptotically converge to the true global values. 
	To validate the proposed strategy, we conduct comprehensive simulations and real-world experiments to evaluate the efficiency of complex shape formation and adaptability to swarm-size variations. 
\end{abstract}

\begin{IEEEkeywords}
Shape formation, meanshift control, discrete mass distribution, robot swarm.
\end{IEEEkeywords}


\markboth{}%
{}

\definecolor{limegreen}{rgb}{0.2, 0.8, 0.2}
\definecolor{forestgreen}{rgb}{0.13, 0.55, 0.13}
\definecolor{greenhtml}{rgb}{0.0, 0.5, 0.0}

\section{Introduction}

\IEEEPARstart{S}{hape} formation has received considerable attention in the past two decades due to its broad applications in multi-robot systems, including object transportation \cite{yang2022autonomous}, \cite{li2025UrbanAir}, area monitoring \cite{Qu2024Automated}, and digital painting \cite{alhafnawi2021self}. 
The objective of shape formation is to coordinate a group of robots to autonomously form a desired geometric shape from any initial configuration. 
To achieve this objective, various control strategies have been developed, such as graph-based \cite{Ze2023TIE}, \cite{yang2024joint} and assignment-based \cite{Wang2020Shape}, \cite{Li2025DynamicAssignment} methods.
However, these methods suffer from the curse of dimensionality in large-scale swarms and lack adaptability to swarm-size variations, as their shape representation is fundamentally coupled with the number of robots. 

To achieve shape formation in large-scale swarms that is also adaptable to swarm-size variations, the density-distribution method has emerged as a promising paradigm \cite{Zheng2022Transporting}, \cite{Ito2024MaxEntropy}. 
In this paradigm, the formation shape is modeled as a continuous density function defined over a connected region. 
The shape formation problem is thus transformed into a task of controlling the current swarm density to track the desired density, and control commands for individual robots are generated based on the global density transport strategy.
Such a framework inherently enables strong applicability to large-scale swarms and adaptability to swarm-size variations, as the density distribution is defined independently of the number of robots. 
One typical class of methods in this framework is based on mean-field theory \cite{Zheng2022Transporting}, \cite{Eren2017Velocity}. 
These methods use a mean-field partial differential equation to derive a global velocity field from density transport, which in turn drives the robots. 
Subsequently, the works in \cite{chen2023density}, \cite{Sinigaglia2025RobustOptimal} further employed optimal transport theory to obtain the optimal solution to the mean-field density control problem. 
Another class of methods is based on image moments \cite{Liu2024Self}. 
This approach models the formation shape as a density distribution based on image moments and provides a decentralized control strategy to achieve the desired moments.

Despite their effectiveness in shape formation, existing density-distribution methods still face practical challenges.

The first challenge is to accurately represent a complex shape using a density-distribution function. 
In formation problems, shape representation is an important factor in determining the control strategy. 
Most existing density-distribution methods require the design of a continuous density function to represent the formation shape \cite{Zheng2022Transporting}, \cite{Sinigaglia2025RobustOptimal}. 
However, in practice, formulating rigorous mathematical functions to accurately describe complex shapes, such as a fish-bone shape, is difficult. 
To solve this problem, recent work \cite{Liu2024Self} introduces the concept of image moments to describe the formation shape. 
However, its representation accuracy for complex shapes remains a practical challenge (see Section \ref{Sec_evaluation}). 
Therefore, it is important to design a new representation method that can eliminate the requirement for defining continuous density functions in distribution-based shape formation.

The second challenge is to achieve a decentralized implementation for the density-distribution method. 
Decentralized methods are more adaptable than centralized ones in practice. 
However, most density-distribution methods are only suitable for centralized scenarios \cite{Zheng2022Transporting}, \cite{cui2024density}. 
The primary reason is that the definition of the density distribution itself relies on centralized global positioning data from all robots. 
This positioning data is difficult for an individual robot to access locally. 
Although decentralized density estimation allows for reconstructing the global density distribution, current methods are limited to static scenarios \cite{Gu2008Distributed}, \cite{Battistelli2016Stability}. 
Recently, the works in \cite{Liu2024Self}, \cite{Zheng2022Distributed} have provided decentralized density estimators for time-varying scenarios, where robot motion dynamically affects the formation distribution. 
However, the presence of estimation errors poses a challenge to achieving high control accuracy. 
Therefore, it is crucial to develop a decentralized distribution-based control law for accurate shape formation in large-scale robot swarms. 

Motivated by these unsolved challenges, this paper aims to develop a fully decentralized shape formation control strategy based on a discrete mass distribution. 
More specifically, the contributions and novelties are summarized as follows. 

1)~To accurately represent highly complex formation shapes, we model the formation shape as a discrete mass-distribution function defined over a set of sample points. 
This model is decoupled from the number of robots and is thus inherently adaptable to swarm-size variations. 
Compared to methods using continuous density-distribution functions \cite{Zheng2022Transporting}, \cite{Liu2024Self}, our discrete model reduces the distortion in representing highly complex shapes by eliminating the requirement for defining continuous density functions. 
Our model is also computationally efficient, which stems from its finite-dimensional discrete domain, contrasting with the infinite-dimensional continuum of continuous density functions. 

2)~For decentralized distribution-based shape formation, we design a distribution-similarity error metric to assess the deviation between the swarm distribution and the sample-point distribution. 
Then, we propose a decentralized meanshift-based control law \cite{Sun2025MeanShift} to minimize the error metric, i.e., to drive the swarm distribution to fit the sample-point distribution, by feeding back the mass values of all sample points. 
The mass values are estimated by each robot based on its own position and local interactions in a decentralized consensus manner via our improved mass estimator. 
It is proven that the mass estimates for all sample points can converge to the global true values asymptotically. 

3)~We provide a stability analysis of our proposed meanshift control law. 
This is a valuable contribution since behavior-based formation methods are notoriously difficult to analyze for stability---particularly when integrating collision avoidance \cite{Zheng2022Transporting}, \cite{Zhao2019Bearing}, \cite{Liu2023Distributed}. 
Although convergence in convex cases is theoretically proven in this paper, the result is still meaningful, given that existing work can only guarantee convergence for rectangular shapes \cite{Sun2023Assembly}.
Simulation and experimental results further validate the effectiveness of our proposed control strategy for highly non-convex shapes. 

The remainder of this paper is organized as follows. 
Section~\ref{Sec_Overview} presents the problem formulation and an overview of the system. 
Section~\ref{Sec_statement} introduces the definition of the discrete mass distribution model, the control objective with respect to this model, and the design of a distribution-similarity error metric. 
Section~\ref{Sec_formation} details the decentralized meanshift control strategy and its convergence analysis.
Simulation and experimental results are presented in Section~\ref{Sec_evaluation}. 
Finally, Section~\ref{Sec_conclusion} concludes the paper.

\section{Problem Formulation and Overview}
\label{Sec_Overview}

\subsection{Robot Model}
\label{Subsec_RobotModel}

Consider a swarm of $n$ robots operating in $\mathbb{R}^d$, where $n \geq 2$ and $d \in \{1, 2, 3\}$. 
The position of robot $i$ is denoted by $p_i \in \mathbb{R}^d$ for $i \in \{1, \ldots, n\}$. 
Each robot follows the dynamics $\dot{p}_i = v_i$, where the velocity control input $v_i$ satisfies $\|v_i\| \leq v_{\mathrm{max}}$ with $v_{\mathrm{max}} > 0$. 
If the Euclidean distance $\| p_i-p_j \|$ between robots $i$ and $j$ is less than a sensing range $r_\mathrm{sense} > 0$, they can exchange information via wireless communication. 
The interaction network is modeled as a time-varying undirected graph $\mathcal{G}(t) = (\mathcal{V}, \mathcal{E}(t))$ for $t \geq t_0$, where $\mathcal{V} = \{1, \ldots, n\}$ and $\mathcal{E}(t) = \{ (i,j) \in \mathcal{V} \times \mathcal{V} : \|p_i(t) - p_j(t)\| \leq r_{\mathrm{sense}}, i\ne j\}$. 
The neighbor set of robot $i$ is defined as $\mathcal{N}_i(t) = \{ j \in \mathcal{V} : (i,j) \in \mathcal{E}(t) \}$. 
For clarity, the main symbols used in this paper are listed in Table~\ref{Table_Symbols}.

\begin{table}[H]
	\centering
	\caption{List of Main Symbols Used in This Paper} \label{Table_Symbols}
	\begin{tabular}{cl}
		\toprule
		\multicolumn{1}{c}{\textbf{Symbols}} & \multicolumn{1}{c}{\textbf{Description}} \\
		\midrule
		$p_i$ & Position of robot $i$ \\
		$v_i$ & Control input of robot $i$ \\
		$v_i^\mathrm{ms}$ & Meanshift control command for robot $i$ \\
		$v_i^\mathrm{cv}$ & Collision-avoidance control command for robot $i$ \\
		$q_k$ & Position of sample point $k$ \\
		$P_k$ & Mass of sample point $k$ \\
		$\hat{P}_{k,i}$ & Robot $i$'s mass estimate of sample point $k$ \\
		$q_{o,i}$ & Desired shape position interpreted by robot $i$ \\
		$\theta_{o,i}$ & Desired shape orientation interpreted by robot $i$ \\
		$\|\cdot\|$ & Euclidean norm \\
		$\beta$ & Bandwidth of the Gaussian kernel function \\
		$r_\mathrm{sense}$ & Sensing range of robots \\
		$r_\mathrm{avoid}$ & Collision-avoidance range of robots \\
		$v_\mathrm{max}$ & Maximum speed of robots \\
		$\mathcal{N}_i$ & Set of neighboring robots of robot $i$ \\
		$\mathcal{N}_i^\prime$ & Set of neighbors within robot $i$'s avoidance range \\
		$\mathcal{G}$ & Communication graph of the robot swarm \\
		\bottomrule
	\end{tabular}
\end{table}

\subsection{System Overview}

\begin{figure*}[!t] 
	\centering	\includegraphics[width=1\linewidth]{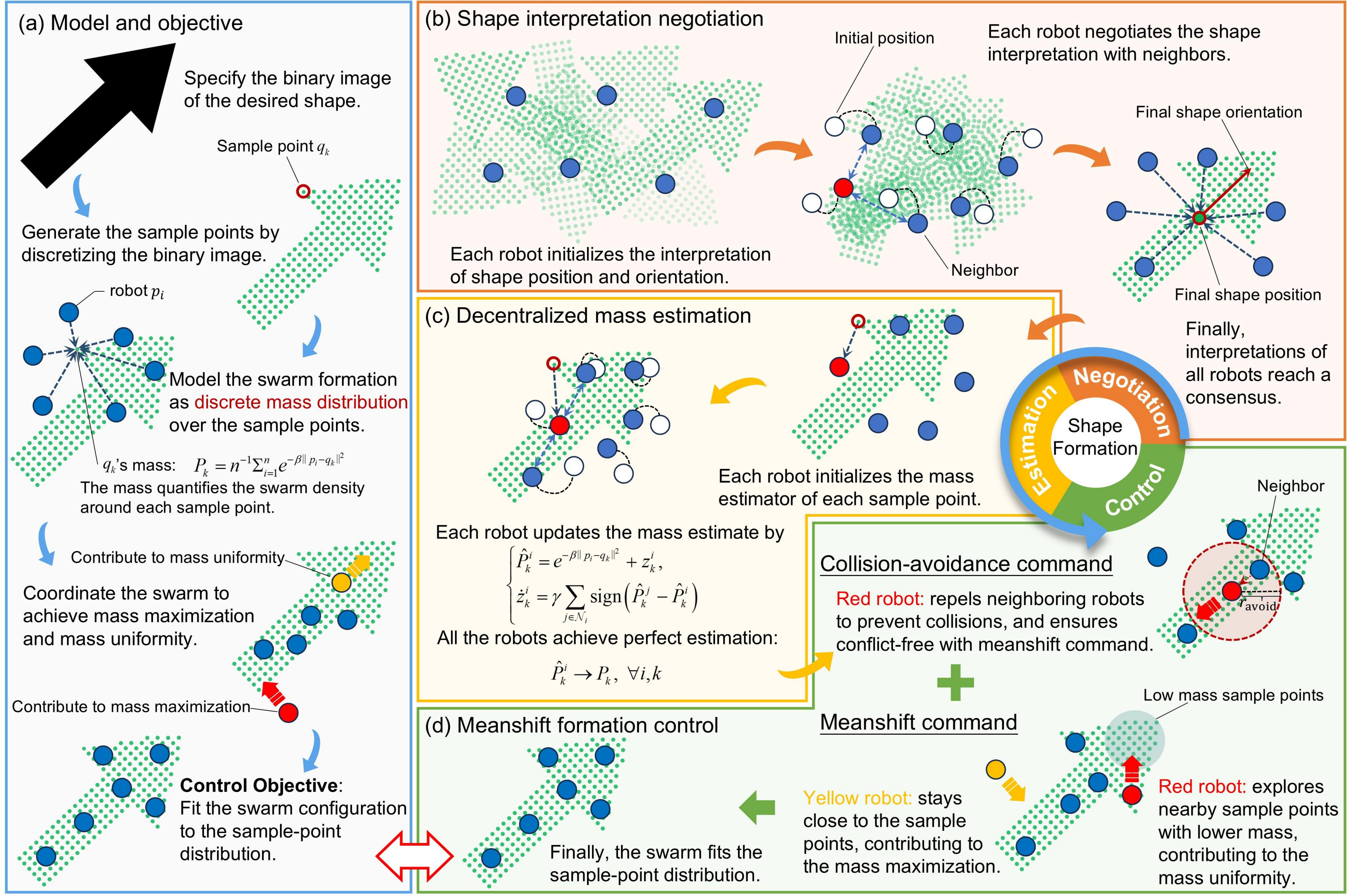}
	\caption{		
		Overview of the system. 
		(a) Illustration of the proposed model and control objective. 
		(b) Intuitive explanation of the shape interpretation negotiation process. 
		(c) Intuitive explanation of the decentralized mass estimation. 
		(d) Intuitive explanation of the meanshift control strategy.
		\label{fig_overview}}		
\end{figure*}

The objective of the swarm is to achieve shape formation using only local communication, starting from an arbitrary initial configuration. 
A set of discrete sample points specifies the desired shape, and the goal is to coordinate the swarm's spatial distribution to fit the sample-point distribution.
To this end, the swarm's spatial distribution is modeled using a mass-distribution representation (see Fig.~\ref{fig_overview}(a)). 
An overview of the proposed strategy is depicted in Fig.~\ref{fig_overview}. 

The first step is to design decentralized negotiation protocols for the shape parameters (see Fig.~\ref{fig_overview}(b)). 
Instead of being predefined, the position and orientation of the desired shape are determined autonomously by the robots. 
Initially, each robot independently determines these parameters based on its local perspective, which may lead to conflicting interpretations. 
Through the proposed consensus protocols, robots resolve these discrepancies via local interactions. 
Finally, all robots reach consensus on the shape interpretations within a finite time. 
The second step is to design a decentralized mass estimator (see Fig.~\ref{fig_overview}(c)). 
Since the mass distribution calculation relies on all robots' positions, a decentralized estimator is employed, where each robot estimates the global mass distribution through local interactions. 
The proposed estimator can theoretically ensure convergence to zero estimation error. 
The final step is to design a formation controller to accomplish shape formation (see Fig.~\ref{fig_overview}(d)). 
The meanshift controller consists of two commands: a meanshift command that guides robots toward nearby sample points with lower mass, and a collision-avoidance command that repulses neighboring robots. 
These commands work collaboratively to coordinate the swarm, enabling it to fit the sample-point distribution and form the desired shape. 
Furthermore, a convergence analysis is provided in Section~\ref{Subsec_analysis} to mathematically assess the system's feasibility and stability.

\section{Distribution Description of Shape Formation}
\label{Sec_statement}

In this section, we first propose a sample-point-based discrete mass distribution model for the formation objective and explore its advantages. 
Then, a distribution-similarity error metric is presented to assess the accomplishment of the formation objective. 
Finally, decentralized parameter negotiation protocols are proposed for robots to autonomously determine the position and orientation of the desired shape. 

\subsection{Mass-Distribution Based Formation Definition}
\label{Subsec_shapedefinition}

In this paper, we use a discrete mass distribution defined over a set of finite sample points to describe the swarm formation, as shown in Fig.~\ref{Fig_CtrlObjective}(a). 
The sample points are specified by a uniform discretization of the desired shape. 
Each sample point is described by a Euclidean parameter $q_k \in\mathbb{R}^d$, i.e., the spatial position of sample point $k\in \{1,...,m\}$. 
The set of all sample points is consistent with the desired shape in terms of both spatial arrangement and statistical properties, and thus it can characterize the geometric configuration of the desired shape. 
There are many benefits to using sample points to represent the desired shape, such as geometric flexibility, swarm scalability, and computational efficiency, which will be discussed later in this subsection.

Once the sample-point set for a desired shape is determined, the control objective is regarded as coordinating the robot swarm's spatial distribution to match the sample-point distribution.
The swarm's spatial distribution is described by a mass function assigning a nonnegative mass value to each sample point, representing the robot density around it.
Specifically, the mass of the $k$-th sample point is defined as 
\begin{align}
	P_k\left(\mathbf{p}\right)=\dfrac{1}{n} \sum_{i=1}^{n} e^{-\beta \left\| q_k - p_i \right\|^2}, \quad k\in \{1,...,m\}
	\label{Equ_massfunction}
\end{align}
where $\mathbf{p}=[p_1^\top,...,p_n^\top]^\top \in \mathbb{R}^{dn}$ refers to the configuration of the robot swarm, and the constant $\beta>0$ is a kernel bandwidth. 
Recall that $p_i$ is robot $i$'s position and $q_k$ is sample point $k$'s position. 
Here, we use the idea of \emph{Kernel Density Estimation} to estimate the spatial distribution of swarm formation, which is a non-parametric method for estimating a spatial distribution from discrete samples in the absence of prior distribution information \cite{Cover2001Wiley}.
As can be seen, the higher $P_k$ is, the greater the density of robots around $q_k$.
In this way, the overall discrete mass distribution (or mass distribution for simplicity) of the robot swarm can be denoted as $\mathbf{P}(\mathbf{p}) = [P_1(\mathbf{p}), \ldots, P_m(\mathbf{p})]^\top \in \mathbb{R}^m$. 

To sum up, our task is to design a decentralized control strategy to coordinate a robot swarm from any initial configuration to fit the sample-point distribution of a given desired shape. 
For this purpose, the objective of the control strategy is to achieve the following two objectives based on the mass distribution model:
\begin{itemize}
	\item \textit{Mass maximization}: For any $k\in \{1,...,m\}$, $P_k$ should be as large as possible. 
	\item \textit{Mass uniformity}: For any $k,l\in \{1,...,m\}$ and $k\neq l$, $P_k$ and $P_l$ should be as consistent as possible. 
\end{itemize}
Next, we use two examples to explain the concepts of mass uniformity and mass maximization, as shown in Fig.~\ref{Fig_CtrlObjective}(b--c). 
The mass maximization seeks to maximize the swarm density around each sample point, ensuring that robots concentrate densely near the desired shape, as demonstrated in Fig.~\ref{Fig_CtrlObjective}(b). 
The mass uniformity, on the other hand, aims to maintain a consistent robot density across all the sample points, promoting a uniform swarm spatial distribution across the desired shape, as shown in Fig.~\ref{Fig_CtrlObjective}(c). 
By balancing these two objectives, the swarm can fit the sample-point distribution. 

\begin{figure}[!t]
	\centering	\includegraphics[width=0.9\linewidth]{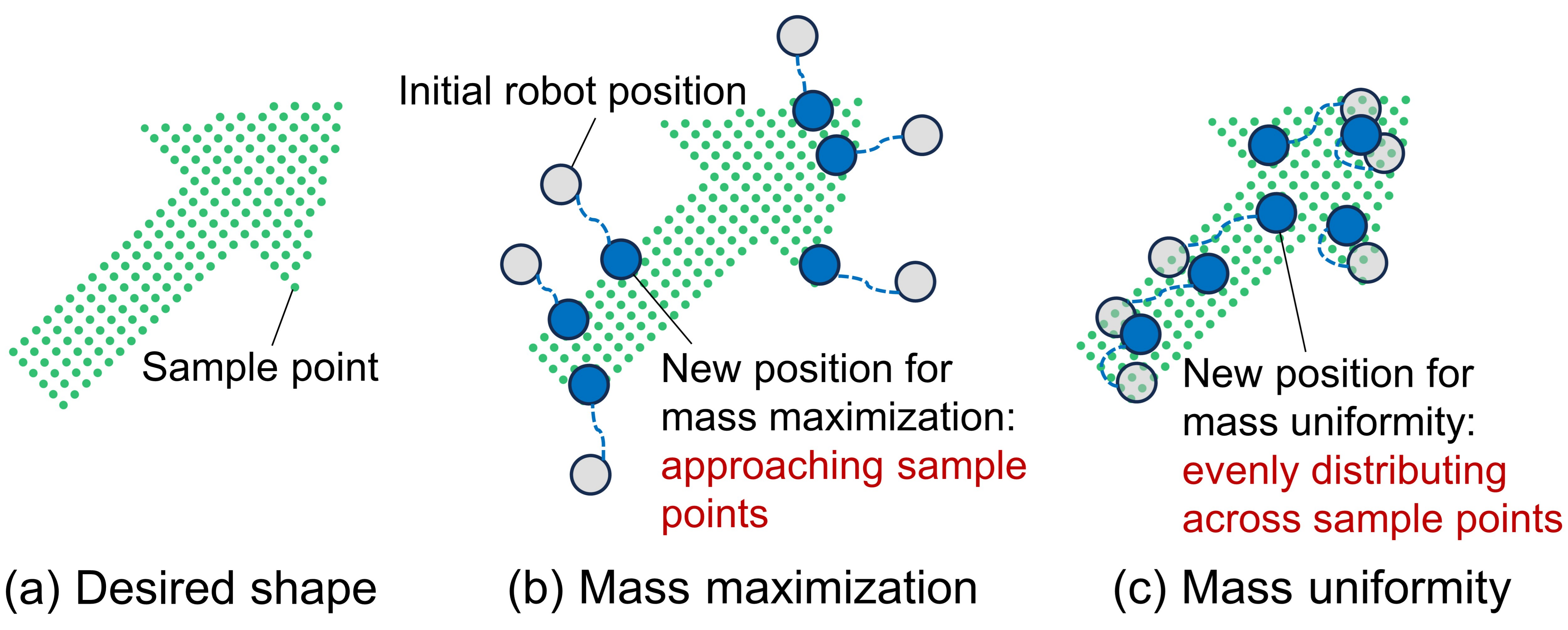}
	\caption{
		Intuitive explanation of the control objectives. 
		(a) Desired shape characterized by sample points. 
		(b) The intuitive explanation of the mass maximization objective. 
		(c) The intuitive explanation of the mass uniformity objective. 
		\label{Fig_CtrlObjective}}
\end{figure}

Some important remarks about the sample-point-based discrete mass distribution model are as follows. 
\begin{itemize}
	\item \textit{Geometric Flexibility}: Representing a shape as a set of sample points greatly simplifies the design of complex geometries. 
	This approach is more flexible than defining shapes via mathematical equations or inequality constraints \cite{Sun2025TIE}, which can become extremely cumbersome for intricate forms. 
	
	\item \textit{Swarm Scalability}: The distribution of sample points is independent of the number of robots, as points are not allocated to specific agents. 
	This feature provides inherent adaptability to swarm-size variations and self-healing capabilities. 
\end{itemize}
Moreover, compared to the state-of-the-art distribution-based methods \cite{Zheng2022Transporting}, \cite{Liu2024Self}, our model exhibits the following advantages.
\begin{itemize}
	\item \textit{Boundary Clarity}: Compared to methods \cite{Zheng2022Transporting}, \cite{Liu2024Self} which can only ambiguously characterize shape boundaries due to smoothness of the density function, the discrete sample points used in the mass distribution model can better represent the shape boundary, ensuring that the desired formation remains within the desired shape boundary.
	\item \textit{Computational Efficiency}: 
	Compared to \cite{Zheng2022Transporting}, which models the shape as a continuous density function defined over an infinite continuum, our discrete model is defined over a finite point set. 
	This results in a finite-dimensional vector representation that is significantly more efficient to compute and communicate. 
\end{itemize}

\subsection{Distribution-Similarity Error Metric of Shape Formation}
\label{Subsec_shapesimilarity}

To assess the deviation between the swarm's spatial distribution and the sample-point distribution, we propose a distribution-similarity error metric as 
\begin{align}
	F\left( \mathbf{p} \right) = F_{\mathrm{max}}\left( \mathbf{p} \right) + F_{\mathrm{uni}}\left( \mathbf{p} \right) 
	\label{Equ_errormetric}
\end{align}
where $F_{\mathrm{max}}$ and $F_{\mathrm{uni}}$ represent the mass maximization metric and the mass uniformity metric, respectively.
The mathematical expressions of the two metrics are defined as follows. 

The metric $F_{\mathrm{max}}$ is used to quantify the achievement of the mass maximization objective and is defined as
\begin{align}
	F_{\mathrm{max}}\left( \mathbf{p} \right) = -\ln \left(\sqrt{\sum_{k=1}^m P_k^2}\right)
	\label{Equ_maximization}
\end{align}
where $m$ is the number of sample points. 
As seen in \eqref{Equ_maximization}, the larger the total mass, the smaller the metric $F_{\mathrm{max}}$.
If robots move closer to the sample points, the increase in mass will cause $F_{\mathrm{max}}$ to decrease.

The metric $F_{\mathrm{uni}}$ evaluates the achievement of the mass uniformity objective and is defined as
\begin{align}
	F_{\mathrm{uni}}\left( \mathbf{p} \right) = -\dfrac{1}{m} \sum_{k=1}^m \ln \left(\sqrt{\dfrac{m P_k^2}{\sum_{l=1}^m P_l^2}}\right)
	\label{Equ_uniformity}
\end{align}
where $P_k,P_l$ denote the mass of sample points $k$ and $l$, respectively. 
As seen from \eqref{Equ_uniformity}, the closer $P_k$ and $P_l$ are, the smaller the metric $F_{\mathrm{uni}}$ is. 
In particular, if $P_k = P_l$, the metric $F_{\mathrm{uni}}$ is zero. 
Moreover, $F_{\mathrm{uni}}$ is a nonnegative function and reaches zero if and only if $P_k = P_l$ for all $k,l\in \{1,...,m\}$, as proven in Lemma \ref{Lem_uniformity}. 

\begin{lemma}
	The metric $F_{\mathrm{uni}}( \mathbf{p} ) \geq 0$ for all $\mathbf{p} \in \mathbb{R}^{dn}$, and $F_{\mathrm{uni}}(\mathbf{p}) = 0$ if and only if $P_k(\mathbf{p}) = P_l(\mathbf{p})$ for all $k,l \in \{1,\ldots,m\}$. 	
	\label{Lem_uniformity}
\end{lemma}
\begin{proof}
	The proof is given in the Appendix. 
\end{proof}

The proposed metric $F$ is continuously differentiable with respect to the position $p_i$ of any robot $i$. 
According to the definition \eqref{Equ_errormetric}, the gradient of $F$ with respect to $p_i$ can be computed as $\nabla_{p_i} F\left(\mathbf{p}\right)=\frac{2\beta}{mn}\sum_{k=1}^m P_k^{-1} e^{-\beta\|p_i-q_k\|^2}(p_i-q_k)$, where $\nabla$ represents the gradient operator. 
Recall that $\beta$ is the kernel bandwidth. 
In this way, the gradient of $F$ with respect to $p$ is given by $\nabla_\mathbf{p} F(\mathbf{p}) = \left[\nabla_{p_1} F^\top(\mathbf{p}), \ldots, \nabla_{p_n} F^\top(\mathbf{p}) \right]^\top$. 
The specific differentiability result of $F$ is given by Lemma~\ref{Lem_Smooth}, which will be used later in Section~\ref{Sec_formation} to design and analyze the control law. 

\begin{lemma}
	The metric $F(\mathbf{p})$ is smooth, i.e., infinitely differentiable with respect to $\mathbf{p}$ on $\mathbb{R}^{dn}$. 
	\label{Lem_Smooth}
\end{lemma}
\begin{proof}
	The proof is given in the Appendix. 
\end{proof}

\begin{figure}[!t]
	\centering	\includegraphics[width=0.9\linewidth]{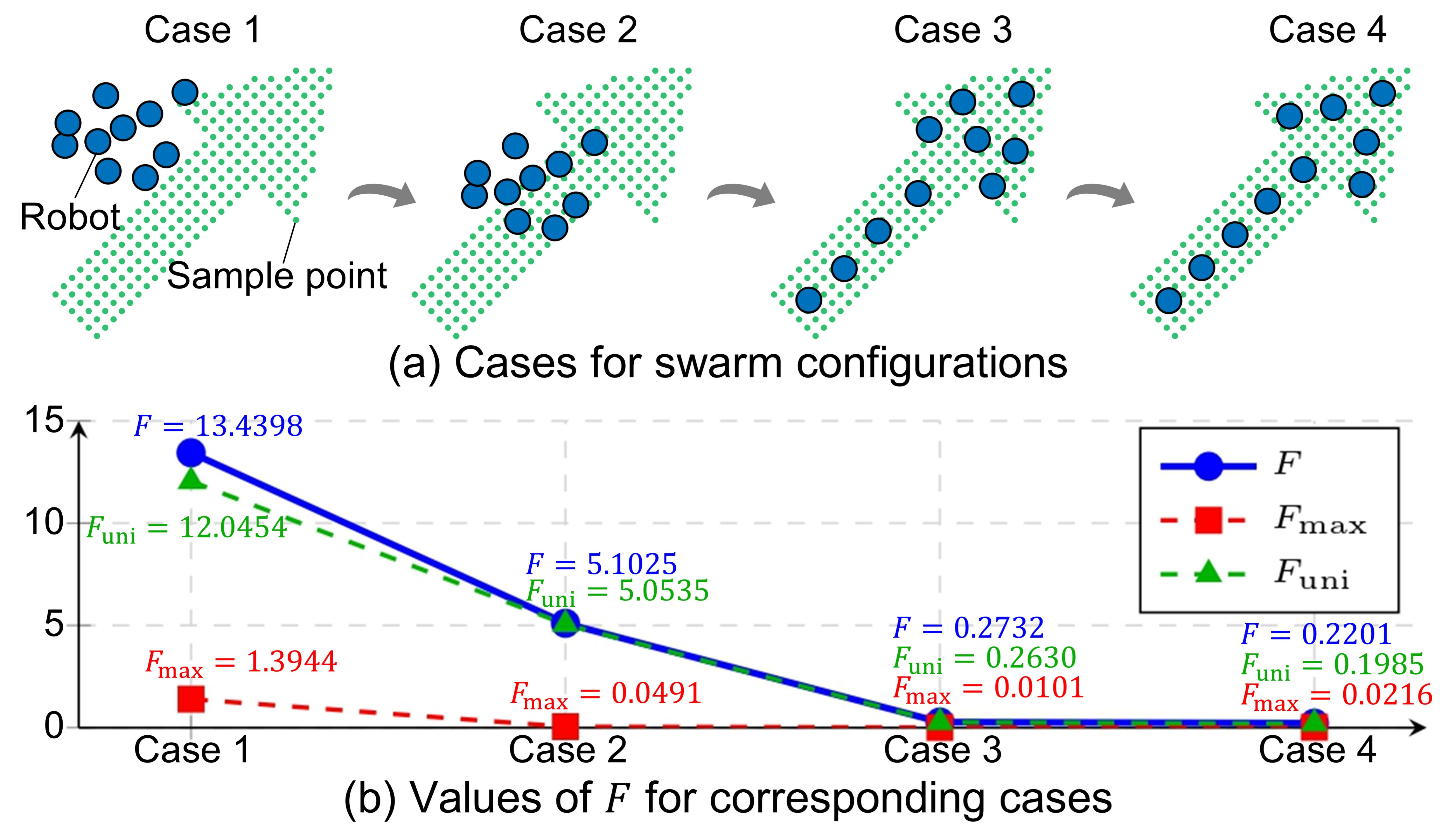}
	\caption{
		Intuitive explanation of the metric $F$. 
		(a) 4 cases of different swarm configurations with respect to the arrow shape. 
		(b) The values of $F, F_{\mathrm{max}}$ and $F_{\mathrm{uni}}$ corresponding to cases shown in (a). 
		\label{Fig_MetricF}}
\end{figure}

Finally, we use an intuitive example to explain the proposed metrics. 
Fig.~\ref{Fig_MetricF}(a) shows different swarm configurations $\mathbf{p}$ for the same desired shape, while Fig.~\ref{Fig_MetricF}(b) plots the corresponding values of $F$, $F_{\mathrm{max}}$, and $F_{\mathrm{uni}}$. 
In case 2, the robots are closer to the sample points than in case 1, resulting in a lower $F_{\mathrm{max}}$ value. 
Additionally, as the swarm distribution becomes more uniform from case 2 to case 4, $F_{\mathrm{uni}}$ decreases. 
Consequently, the similarity between the swarm's spatial distribution and that of the sample points increases from case 1 to case 4, and $F$ decreases accordingly. 
This demonstrates that $F$ effectively measures distribution similarity, and thus the control objective can be formulated as minimizing $F$.

\subsection{Position and Orientation of Formation Shape}
\label{Subsec_shapeparameters}

The position and orientation are two fundamental parameters for shape formation. 
To determine the absolute coordinates $q_k$ of each sample point, the swarm must agree on the shape's position and orientation.
In this paper, these parameters are negotiated by the robots in a decentralized manner, rather than being directly provided to all robots. 
Specifically, each robot initially interprets both parameters from its local perspective, which may conflict with those of other robots.
Then, the robots resolve the conflicts through local negotiation and reach consensus on the final position and orientation of the formation shape.

The position of a shape is defined as the position of the sample point that is closest to the center of the shape.
Denote $q_{o,i}$ as the position of the desired formation shape interpreted by robot $i$, for $i\in \{1,\ldots,n\}$.
Different robots' interpretations of $q_o$ can reach a consensus by the following decentralized protocol: 
\begin{align}
	\dot q_{o,i}=-c_1 \sum_{j \in \mathcal{N}_i} \mathrm{sign} \left(q_{o,i}-q_{o,j}\right)\left|q_{o,i}-q_{o,j}\right|^{\alpha}
	\label{Equ_protocol_pos}
\end{align}
where $c_1>0$ and $0<\alpha<1$ are two constants. 
The operators $\mathrm{sign}(\cdot)$ and $|\cdot|$ represent the signum and absolute functions defined component-wise, respectively. 
Recall that $\mathcal{N}_i$ is the neighbor set of robot $i$. 
Initially, each robot interprets the position of itself as the shape position, i.e., $q_{o,i}(t_0)=p_i(t_0)$. 
As seen from \eqref{Equ_protocol_pos}, the term in the sum symbol is the discrepancy in the interpretation of $q_{o,i}$ between robot $i$ and its neighboring robots, which aims to drive $q_{o,i}\rightarrow q_{o,j}$. 
Moreover, since $0<\alpha<1$, the negotiation consensus can be achieved in a finite time. 
The convergence analysis of the protocol \eqref{Equ_protocol_pos} is given in Theorem~\ref{Thm_Nego_Consensus}.

The orientation of a shape is the angle between a pre-defined orientation vector and the positive $x$-axis, where the orientation vector starts from the center of the shape and its direction can be arbitrarily selected.
Let $\theta_{o,i}$ denote the orientation angle of the desired formation shape interpreted by robot $i$.
The decentralized interpretations of $\theta_{o,i}$ can be governed by the following negotiation protocol:
\begin{align}
	\dot \theta_{o,i}=-c_2 \sum_{j \in \mathcal{N}_i} \mathrm{sign} \left(\theta_{o,i}-\theta_{o,j}\right)\left|\theta_{o,i}-\theta_{o,j}\right|^{\alpha}
	\label{Equ_protocol_ori}
\end{align}
where $c_2$ is a positive constant. 
The initial value of $\theta_{o,i}$ can be random or follow task orientation requirements.
As in protocol \eqref{Equ_protocol_pos}, the summed term in \eqref{Equ_protocol_ori} is also the discrepancy in the interpretation of $\theta_o$ between robot $i$ and its neighbors, driving $\theta_{o,i}\rightarrow \theta_{o,j}$.
In both position and orientation negotiations, each robot achieves consensus on the final formation shape using its neighbors' interpretations via wireless communication.

Next, we present the convergence analysis of protocols \eqref{Equ_protocol_pos} and \eqref{Equ_protocol_ori}. 
To this end, we first make the following assumption.

\begin{assumption}
	$\mathcal{G}(t)$ is connected for all $t \ge t_0$. 
	\label{Asm_graph}
\end{assumption}

Assumption~\ref{Asm_graph} is common in shape formation control \cite{Sun2023Assembly}, \cite{Nguyen2019TRO}. 
As an alternative, state-of-the-art methods on network-connectivity preservation can rigorously ensure that Assumption~\ref{Asm_graph} holds \cite{Gasparri2017TRO}, \cite{Li2023TNSE}. 
Now, we are ready to present the convergence analysis of protocols \eqref{Equ_protocol_pos} and \eqref{Equ_protocol_ori} in Theorem~\ref{Thm_Nego_Consensus}.

\begin{theorem}\label{Thm_Nego_Consensus}
	Suppose that $\bar q_o=\frac{1}{n}\sum \nolimits_{i=1}^n q_{o,i}(t_0)$ and $\bar \theta_o=\frac{1}{n}\sum \nolimits_{i=1}^n \theta_{o,i}(t_0)$. 
	If $\mathcal{G}(t)$ is connected for all $t \ge t_0$, then under protocols \eqref{Equ_protocol_pos} and \eqref{Equ_protocol_ori}, there exists a finite time $t^{\ast} \ge t_0$ such that $q_{o,i}(t) = \bar q_o$ and $\theta_{o,i}(t) = \bar \theta_o$ for all $t \geq t^*$. 
\end{theorem}
\begin{proof}
	The proof is given in the Appendix. 
\end{proof}

\section{Meanshift Strategy for Shape Formation}
\label{Sec_formation}

To achieve the formation objective specified by the discrete mass distribution model, we now propose a decentralized meanshift strategy. 
Specifically, we first present a decentralized mass estimator that enables each robot to estimate the mass of all the sample points through local interactions. 
Then, we design a decentralized mass-feedback controller that coordinates the spatial distribution of the robot swarm to fit the sample-point distribution. 
Finally, we provide a convergence analysis of the proposed strategy. 

\subsection{Decentralized Mass Estimation}
\label{Subsec_mass_estimation}

Given a desired shape specified by sample points, the first step is to obtain the current mass at each sample point. 
The mass of each sample point changes dynamically as the robot swarm moves. 
A straightforward approach is to broadcast the mass distribution to all robots.
The limitation is that a centralized node is required to calculate the entire mass distribution, creating a single point of failure.
Instead, our approach is to let the robots estimate the mass in a decentralized manner. 

The mass value $P_k$ denotes the density of all the robots around sample point $k$, as defined in \eqref{Equ_massfunction}. 
Suppose that $\hat{P}_{k,i}$ represents robot $i$'s estimate of $P_k$ from its local perspective. 
The mass estimate $\hat{P}_{k,i}$ of each robot can converge to $P_k$ by the following decentralized estimator: 
\begin{align}
	\begin{split}
		\hat{P}_{k,i} &= e^{-\beta \| p_i - q_k \|^2} + z_{k,i} \\
		\dot{z}_{k,i} &= \gamma \sum_{j \in \mathcal{N}_i} \operatorname{sign}\left(\hat{P}_{k,j} - \hat{P}_{k,i} \right)
	\end{split}
	\label{Equ_massestimator}
\end{align}
where $\gamma > 0$ is a constant gain, and $z_{k,i} \in \mathbb{R}$ is an internal state initialized as $z_{k,i}(t_0) = 0$. 
Recall that $\beta > 0$ denotes the kernel bandwidth. 
As seen from \eqref{Equ_massestimator}, the term in the sum symbol represents the discrepancy between robot $i$'s mass estimate and those of its neighboring robots, which aims to drive $\hat{P}_{k,i} \rightarrow \hat{P}_{k,j}$. 
Moreover, the term $e^{-\beta\|p_i - q_k\|^2}$ is a reference signal, which is used to ensure that $\hat{P}_{k,i}$ converges to the average of all robots' reference signals, i.e.,
$\hat{P}_{k,i} \rightarrow \frac{1}{n}\sum_{j=1}^n e^{-\beta\|p_j - q_k\|^2}$.
Next, we present the convergence analysis of the proposed estimator \eqref{Equ_massestimator} in Theorem~\ref{Thm_Esti_Converge}.

\begin{theorem}\label{Thm_Esti_Converge}
	Suppose that each robot's position is arbitrarily initialized and its speed is bounded by a maximum speed $v_{\mathrm{max}} > 0$. 
	Under Assumption~\ref{Asm_graph}, if the parameter $\gamma$ satisfies $\gamma > (n-1)\sqrt{2\beta/e} \, v_{\mathrm{max}}$, then for each $i \in \{1, \ldots, n\}$ and $k \in \{1, \ldots, m\}$, the mass estimator $\hat{P}_{k,i}$ defined in \eqref{Equ_massestimator} asymptotically converges to the true mass $P_k$ defined in \eqref{Equ_massfunction}.
\end{theorem}
\begin{proof}
	The proof is given in the Appendix. 
\end{proof}

Note that the velocity of each robot is bounded by a saturation function, which will be presented later in Section~\ref{Subsec_controller}. 

\subsection{Meanshift Controller via Mass Feedback}
\label{Subsec_controller}	

To implement the fitting of the swarm's distribution to the given sample points, we design the following decentralized controller:
\begin{align}  
	v_i=\mathrm{sat}\left(v_i^\mathrm{ms}+v_i^\mathrm{cv}\right)
	\label{Equ_controllaw}
\end{align}
where $v_i^\mathrm{ms}$ is the meanshift velocity command, and $v_i^\mathrm{cv}$ is the collision-avoidance velocity command. 
Here, the saturation function $\mathrm{sat}(\cdot)$ is given by
\begin{align}\label{Equ_SatFcn}
	\mathrm{sat}(z) = \begin{cases}
		\dfrac{z}{\|z\|}v_{\mathrm{max}}, & \|z\| > v_{\mathrm{max}}\\
		z, & \text{otherwise}
	\end{cases}, \forall z \in \mathbb{R}^d
\end{align}
where $v_\mathrm{max}>0$ is the maximum speed of the robot. 
The roles of the two velocity commands are illustrated as follows. 

\subsubsection{Meanshift Velocity Command}

The meanshift command $v_i^\mathrm{ms}$ in \eqref{Equ_controllaw} aims to drive the robot swarm to fit the sample-point distribution by decreasing the distribution-similarity error metric $F$. 
To this end, $v_i^\mathrm{ms}$ is designed as 
\begin{align}
	v_i^\mathrm{ms} = - \kappa_1  \nabla_{p_i} F\left(\mathbf{p}\right)
	\label{Equ_mscommand}
\end{align}
with a positive gain $\kappa_1=\sigma_1 n / [2\beta \sum_{k=1}^m (P_k)^{-1}\omega(\|q_k-p_i\|)]$, where $\sigma_1$ is a positive constant. 
Recall that $P_k$ is the mass value of sample point $k$ and $q_k$ is the position of the $k$-th sample point. 
Here, the weight function $\omega(z):\mathbb{R}_{\ge 0} \rightarrow \mathbb{R}_{>0}$ is given by $\omega(z)=e^{-\beta z^2}$ where $\beta>0$ denotes the kernel bandwidth. 
This function is monotonically decreasing.
As a result, the weight $\omega(\|q_k-p_i\|)$ is large when the distance between robot $i$ and sample point $k$ is small. 
As can be seen, $n$, $P_k$, and $F(\mathbf{p})$ are three global quantities and thus the definition of $v_i^\mathrm{ms}$ in \eqref{Equ_mscommand} is global. 
Next, we localize the definition \eqref{Equ_mscommand} from the perspective of the individual robot.

By taking the gradient of $F$ and replacing $P_k$ with $\hat{P}_{k,i}$, the decentralized form of $v_i^\mathrm{ms}$ from robot $i$'s perspective is 
\begin{align}
	v_i^\mathrm{ms} = \dfrac{\frac{\sigma_1}{m}\sum_{k=1}^{m}(\hat{P}_{k,i})^{-1}\omega\left(\left\|q_k-p_i\right\|\right)\left(q_k-p_i\right)}{\sum_{k=1}^{m}(\hat{P}_{k,i})^{-1}\omega\left(\left\|q_k-p_i\right\|\right)} .
	\label{Equ_mscommand_mod}
\end{align}
As seen in \eqref{Equ_mscommand_mod}, $v_i^\mathrm{ms}$ follows a modified meanshift algorithm, which will be specifically explained later.
Let $\psi(\hat{P}_{k,i},\|q_k-p_i\|)=(\hat{P}_{k,i})^{-1}\omega(\|q_k-p_i\|)$ be the weight function of sample point $k$. 
This function is affected by both the density of robots around point $k$ and the distance between robot $i$ and point $k$. 
Specifically, the closer robot $i$ is to point $k$, the higher $\omega(\|q_k-p_i\|)$ is.
Meanwhile, the fewer other robots are around the point $k$, the greater $(\hat{P}_{k,i})^{-1}$ is. 
As a result, more weights are assigned to the sample points that are closer to robot $i$ and have fewer robots around them. 
This encourages the robots to stay as close to the sample points as possible while guaranteeing that the density of robots around each sample point is as consistent as possible. 
Hence, $v_i^\mathrm{ms}$ can reduce the distribution-similarity error metric $F$, as proved later in Theorem~\ref{Thm_CtrlConvergence}.

The meanshift algorithm is a non-parametric iterative technique for mode-seeking and density estimation, originally developed in machine learning and computer vision. 
Its core principle involves shifting cluster centers toward the weighted average of neighboring points in each iteration. 
In multi-robot systems, it is used to design control laws \cite{Sun2025MeanShift}, \cite{Sun2023Assembly}, enabling robots to move toward the weighted average center of reference positions, expressed as $ v_i = (\sum_{k} w_k q_k)/(\sum_{k} w_k) $ with \( w_k > 0 \) representing the weight for point \( q_k \). 
Our approach enhances this framework by integrating our mass distribution model, where mass and robot-point distance influence the weights.

\subsubsection{Collision-Avoidance Velocity Command}

The collision-avoidance command $v_i^\mathrm{cv}$ prevents collisions with neighboring robots while remaining conflict-free with the shape formation command $v_i^\mathrm{ms}$. 
It is specifically defined as
\begin{align}\label{Equ_CollAvdCmd}
	v_i^\mathrm{cv} = \kappa_2 \tilde{v}_i^\mathrm{cv}
\end{align}
where 
\begin{align*}
	\tilde{v}_i^\mathrm{cv} = \sigma_2 \sum_{
		j \in \mathcal{N}_i^{\prime}} \dfrac{r_\mathrm{avoid} - \|p_i - p_j\|}{\|p_i - p_j\| + \varepsilon}(p_i - p_j)
\end{align*}
is a repulsive term, and the self-tuning gain $\kappa_2$ is given by
\begin{align*}
	\kappa_2 = \begin{cases}
		\varphi, & \left( v_i^\mathrm{ms} \right)^\top \tilde{v}_i^\mathrm{cv} \ge 0 \\
		\varphi \cdot \min \left\{ -\dfrac{(1-\varepsilon) \| v_i^\mathrm{ms} \|^2}{\left( v_i^\mathrm{ms} \right)^\top \tilde{v}_i^\mathrm{cv}}, 1\right\}, & \left( v_i^\mathrm{ms} \right)^\top \tilde{v}_i^\mathrm{cv} < 0.
	\end{cases}
\end{align*}
Here, $\sigma_2>0$ and $0<\varepsilon<1$ are constants, $r_\mathrm{avoid} > 0$ is the collision-avoidance range, and $\mathcal{N}_i^{\prime} = \{ j \in \mathcal{N}_i : \|p_i-p_j\| \le r_\mathrm{avoid} \}$ where $r_\mathrm{avoid} \in (0, r_\mathrm{sense})$ is a collision-avoidance range. 
The coefficient $\varphi = \min \{ \|v_i^\mathrm{ms}\|^2 /\varepsilon, 1 \}$ ensures that $\kappa_2$ is locally Lipschitz at $v_i^\mathrm{ms} = 0$. 
The repulsive term $\tilde{v}_i^\mathrm{cv}$ decreases with increasing inter-robot distance and vanishes when the distance is beyond $r_\mathrm{avoid}$.

The gain $\kappa_2$ is designed to prevent $v_i^\mathrm{cv}$ from opposing $v_i^\mathrm{ms}$. 
In the first case, when $( v_i^\mathrm{ms} )^\top \tilde{v}_i^\mathrm{cv} \ge 0$, the two commands are aligned, so $\kappa_2 = \varphi$ and no adjustment is needed. 
In the second case, when $( v_i^\mathrm{ms} )^\top \tilde{v}_i^\mathrm{cv} < 0$, the commands are opposed, and $\kappa_2$ is reduced to satisfy the conflict-free condition:
\begin{align}
	(v_i^\mathrm{ms} + v_i^\mathrm{cv})^\top v_i^\mathrm{ms} \ge \varepsilon \|v_i^\mathrm{ms}\|^2.
	\label{Equ_ConflictFree}
\end{align}
Lemma~\ref{Lem_ConflictFree} proves that this condition holds, which is used in the convergence analysis of Theorem~\ref{Thm_CtrlConvergence}.

\begin{lemma}\label{Lem_ConflictFree}
	The inequality \eqref{Equ_ConflictFree} always holds. 
\end{lemma}
\begin{proof}
	The proof is given in the Appendix. 
\end{proof}

The computational complexity of the control strategy for each robot is independent of the total swarm size $n$. 
Specifically, the computational cost per robot per control cycle is $O(m + |\mathcal{N}_i'|)$, derived from three primary operations:
1) Decentralized mass estimation: Each robot exchanges and updates mass estimates for all $m$ sample points with its neighbors.
2) Meanshift command computation: The robot calculates contributions from all $m$ sample points to determine its movement direction.
3) Collision-avoidance command: The robot computes the relative positions of all neighboring robots $j \in \mathcal{N}_i'$. 
Thus, the algorithm scales efficiently as the swarm grows.

\subsection{Convergence Analysis}
\label{Subsec_analysis}

In this section, we present the convergence analysis of the proposed controller \eqref{Equ_controllaw}. 
To this end, we make the following assumptions on the mass estimate and the desired shape.

\begin{assumption}
	It is assumed that each robot has a perfect estimate of the mass distribution, i.e., 
	$\hat{P}_{k,i} = P_k$ for all $i \in \{1,\ldots,n\}$ and $k\in \{1,\ldots,m\}$.
	\label{Asm_Pki=Pk}
\end{assumption}

Assumption~\ref{Asm_Pki=Pk} is satisfied if the local mass estimate $\hat{P}_{k,i}$ of each robot reaches a consensus and converges to the true mass value $P_k$. 
The proposed mass estimator \eqref{Equ_massestimator} can achieve this assumption, which has been proved in Theorem~\ref{Thm_Esti_Converge}. 

\begin{assumption}
	The desired shape is assumed to be convex.
	\label{Asm_ShapeConvex}
\end{assumption}

Assumption~\ref{Asm_ShapeConvex} ensures that the convex hull of all sample points coincides with the desired shape.
Without this assumption, the convergence analysis is still valid but only guarantees that all robots achieve stability (as shown in Theorem~\ref{Thm_CtrlConvergence}) and converge to the convex hull of the sample points.
In fact, it is difficult to analyze the convergence of behavior-based formation methods such as \eqref{Equ_controllaw} \cite{Liu2023Distributed},  \cite{cheah2009region}. 
Our previous work has demonstrated the convergence for forming rectangular shapes \cite{Sun2023Assembly}. 
In the following, we extend this convergence result to arbitrary convex shapes in Theorem~\ref{Thm_Converge2Shape}. 

Now, we are ready to present the main result of our work in Theorem~\ref{Thm_CtrlConvergence}. 
The idea of the proof is to utilize LaSalle's Invariance principle (see Lemma~8) to show that the swarm configuration $\mathbf{p}$ can converge to the set of critical points of $F(\mathbf{p})$, denoted as $	\mathcal{D} = \{ \mathbf{p} \in \mathbb{R}^{dn} : \nabla_\mathbf{p} F(\mathbf{p}) = 0 \}$. 

To apply LaSalle's invariance principle, we first show that the swarm trajectory $\mathbf{p}(t)$ is bounded, as stated in Lemma~\ref{Lem_p_bounded}.

\begin{lemma}\label{Lem_p_bounded}
	Under Assumption~\ref{Asm_Pki=Pk}, by the action of controller \eqref{Equ_controllaw}, the trajectory of the swarm $\mathbf{p}(t)$ starting from any initial state $\mathbf{p}(t_0) \in \mathbb{R}^{dn}$ is bounded for all $t \ge t_0$. 
\end{lemma}
\begin{proof}
	The proof is given in the Appendix.
\end{proof}

Then, we show that $v_i$ is locally Lipschitz in $\mathbf{p}$ by Lemma~\ref{Lemma_Vi_Lipschitz}. 
For this purpose, we first demonstrate that $\mathrm{sat}(\cdot)$, $v_i^\mathrm{ms}$ and $v_i^\mathrm{cv}$ are locally Lipschitz by Lemmas~\ref{Lem_MS_Lipschitz} and \ref{Lem_CV_Lipschitz}.

\begin{lemma}\label{Lem_MS_Lipschitz}
	Under Assumption~\ref{Asm_Pki=Pk}, the command $v_i^\mathrm{ms}$ defined in \eqref{Equ_mscommand} is locally Lipschitz in $\mathbf{p} \in \mathbb{R}^{dn}$ for all $i\in \{1, \ldots ,n \}$. 
\end{lemma}
\begin{proof}
	The proof is given in the Appendix.
\end{proof}

\begin{lemma}\label{Lem_CV_Lipschitz}
	The collision-avoidance command $v_i^\mathrm{cv}$ defined in \eqref{Equ_CollAvdCmd} is locally Lipschitz in $\mathbf{p} \in \mathbb{R}^{dn}$ for all $i\in \{1, \ldots ,n \}$.
\end{lemma}
\begin{proof}
	The proof is given in the Appendix.
\end{proof}

\begin{lemma}\label{Lemma_Vi_Lipschitz}
	Under Assumption~\ref{Asm_Pki=Pk}, the command $v_i$ defined in \eqref{Equ_controllaw} is locally Lipschitz continuous with respect to $\mathbf{p} \in \mathbb{R}^{dn}$ for all $i\in \{1, \ldots ,n \}$.
\end{lemma}
\begin{proof}
	The proof is given in the Appendix.
\end{proof}

With the above preparation, the convergence of the decentralized controller \eqref{Equ_controllaw} is analyzed in the following theorems. 

\begin{theorem}\label{Thm_CtrlConvergence}
	Under Assumption~\ref{Asm_Pki=Pk}, by the action of \eqref{Equ_controllaw}, $\mathbf{p}(t)$ converges to $\mathcal{D}$ as $t \rightarrow \infty$ from any initial state $\mathbf{p}(t_0) \in \mathbb{R}^{dn}$.
\end{theorem}
\begin{proof}
	Define 
	\begin{align*}
		\Gamma(u) = \begin{cases}
			\dfrac{v_\mathrm{max}}{u}, & u \ge v_\mathrm{max}\\
			1, & 0 \le u < v_\mathrm{max}
		\end{cases}
	\end{align*}
	for $u \ge 0$. 
	Then $v_i$ defined in \eqref{Equ_controllaw} can be reformulated as 
	\begin{align}
		v_i = \Gamma(\| v_i^\mathrm{ms} + v_i^\mathrm{cv}\|)(v_i^\mathrm{ms} + v_i^\mathrm{cv}). 
		\label{Equ_CtrlLaw_mod}
	\end{align}
	Recall that each robot knows the actual mass distribution, i.e., $\hat{P}_{k,i} = P_k$ by Assumption~\ref{Asm_Pki=Pk}. 
	Thus, the decentralized meanshift command $v_i^\mathrm{ms}$ can be formulated in the centralized form \eqref{Equ_mscommand}, i.e., $v_i^\mathrm{ms} = -\kappa_1(p_i, \mathbf{P}) \nabla_{p_i} F(\mathbf{p})$ where $\kappa_1(p_i, \mathbf{P})$ is a positive self-tuning coefficient determined by $p_i$ and $\mathbf{P}$.    
	
	Substituting \eqref{Equ_CtrlLaw_mod} into the time derivative of $F$ yields
	\begin{align}
		\label{Equ_DotFNegtive}
		\begin{split}
			\dot{F} &= \sum_{i=1}^n \nabla_{p_i}F^\top \Gamma(\| v_i^\mathrm{ms} + v_i^\mathrm{cv}\|)\left(v_i^\mathrm{ms} + v_i^\mathrm{cv} \right)  \\
			&= - \sum_{i=1}^n \Gamma(\left\| v_i^\mathrm{ms} + v_i^\mathrm{cv} \right\|) \dfrac{ (v_i^\mathrm{ms})^\top}{\kappa_1 \left(p_i,\mathbf{P} \right)} (v_i^\mathrm{ms} + v_i^\mathrm{cv})  \\
			&\le  - \sum_{i=1}^n \Gamma(\| v_i^\mathrm{ms} + v_i^\mathrm{cv}\|)  \dfrac{\varepsilon \|v_i^\mathrm{ms}\|^2}{ \kappa_1 \left(p_i,\mathbf{P} \right) }\\
			&= - \varepsilon \sum_{i=1}^n \Gamma(\| v_i^\mathrm{ms} + v_i^\mathrm{cv}\|)\kappa_1 \left(p_i, \mathbf{P} \right) \|\nabla_{p_i}F\|^2 
		\end{split}	        
	\end{align}
	where the inequality is derived in Lemma~\ref{Lem_ConflictFree}. 
	Since $\Gamma, \kappa_1 > 0$, \eqref{Equ_DotFNegtive} implies
	\begin{align*}
		\dot{F}(\mathbf{p}) = 0 
		\;\Leftrightarrow\; \nabla_{p_i} F(\mathbf{p}) = 0 \; \forall i 
		\;\Leftrightarrow\; \nabla F(\mathbf{p}) = 0 
		\;\Leftrightarrow\; \mathbf{p} \in \mathcal{D}.
	\end{align*}
	As a result, $\mathcal{D} = \{ \mathbf{p} \in \mathbb{R}^{dn} : \dot{F}(\mathbf{p}) = 0 \}$.
	
	Next, we apply LaSalle's invariance principle (Lemma~\ref{Lem_LasInvPrin}) by considering $F$ as the Lyapunov function candidate. 
	The system dynamic is $\dot{\mathbf{p}} = f(\mathbf{p})$, where $f(\mathbf{p}) = [v_1^\top, \ldots, v_n^\top]^\top$ is locally Lipschitz in $\mathbf{p}$ since each $v_i$ is locally Lipschitz in $\mathbf{p}$ by Lemma~\ref{Lemma_Vi_Lipschitz}. 
	Let $\Omega = \{ \mathbf{p} \in \mathbb{R}^{dn} : \|\mathbf{p}\| \leq M \}$ be the compact positively invariant set, where $M$ bounds the swarm trajectory $\mathbf{p}$ for $i\in \{1,\ldots,n\}$ by Lemma~\ref{Lem_p_bounded}. 
	The Lyapunov function candidate $F$ is differentiable by Lemma~\ref{Lem_Smooth} and satisfies $\dot{F}\le 0$ via \eqref{Equ_DotFNegtive}. 
	Thus, by Lemma~\ref{Lem_LasInvPrin}, $\mathbf{p}$ converges to the largest invariant subset in $\{ \mathbf{p} \in \mathbb{R}^{dn} : \dot{F}(\mathbf{p}) = 0 \}$, so $\mathbf{p} \rightarrow \mathcal{D}$.
\end{proof}

\begin{theorem}
	\label{Thm_Converge2Shape}
	Under Assumption~\ref{Asm_Pki=Pk} and \ref{Asm_ShapeConvex}, by the action of the controller \eqref{Equ_controllaw}, all robots converge to the desired shape.
\end{theorem}

\begin{proof}
	Theorem~\ref{Thm_CtrlConvergence} suggests that $\mathbf{p}$ converges to $\mathcal{D} = \{ \mathbf{p} \in \mathbb{R}^{dn} : \nabla_\mathbf{p} F(\mathbf{p}) = 0 \}$. 
	By \eqref{Equ_mscommand}, $\nabla_\mathbf{p} F(\mathbf{p}) = 0$ implies $v_i^\mathrm{ms} = 0$, which leads to 
	\begin{align}
		p_i = \dfrac{\sum_{k=1}^m \alpha_k q_k }{\sum_{k=1}^m \alpha_k} 
		\label{Equ_PiConvex}
	\end{align}  
	where $\alpha_k = (P_k)^{-1} e^{-\beta \| p_i - q_k \|^2} > 0$.  
	Equation \eqref{Equ_PiConvex} demonstrates that $p_i$ is a convex combination of the sample points.  
	Note that all sample points lie within the desired shape and, by Assumption~\ref{Asm_ShapeConvex}, the desired shape is convex. 
	As a result, the convex combination of the sample points must be contained in the desired shape. 
	Therefore, $p_i$ is contained in the desired shape, implying that all robots converge to the desired shape. 
\end{proof}

\section{Performance Evaluation}
\label{Sec_evaluation}

In this section, we evaluate the performance of the proposed strategy in both simulations and experiments. 
First, we introduce the performance metrics that we consider. 
Then, we present the simulation results to analyze the performance of the proposed strategy.
Finally, we present experimental results to verify the feasibility of the proposed method in practical applications.

\subsection{Performance Metrics}
\label{Subsec_metric}

The first metric, mass estimation error, evaluates the accuracy of the mass estimator. 
The metric is defined as $E_{\mathrm{est}} = \max_{1 \le i \le n,1 \le k \le m} | \hat{P}_{k,i} - P_k |$, where $\hat{P}_{k,i}$ is the mass estimate and $P_k$ is the actual mass. 
As can be seen,	$E_{\mathrm{est}} = 0$ means that all robots achieve perfect mass estimates for each sample point, i.e., $\hat{P}_{k,i} = P_k$ for all $i \in \{1,\ldots,n\}$ and $k \in \{1,\ldots,m\}$. 
Otherwise, $E_{\mathrm{est}} > 0$. 

The second metric, distribution uniformity, quantifies the spatial uniformity of the robot swarm. 
It is defined as $M_{\mathrm{uni}} = \sum_{i=1}^{n} ( r_{\min,i} - \bar{r}_{\min} )^2$, where $r_{\min,i} = \min_{j \in \mathcal{N}_i} \lVert p_i - p_j \rVert$ denotes the minimum distance from robot $i$ to its neighbors, and $\bar{r}_{\min} = \frac{1}{n} \sum_{i=1}^{n} r_{\min,i}$ is the average minimum distance across all robots. 
Perfect uniformity occurs when all robots have the same minimum neighbor distance, resulting in $M_{\mathrm{uni}} = 0$. 
Otherwise, $M_{\mathrm{uni}} > 0$. 

The third metric, coverage rate, evaluates the proportion of the desired shape area covered by the robots. 
The coverage region for each robot is defined as a circle of radius $r_{\mathrm{cover}}$ centered at its position, where $r_{\mathrm{cover}}$ is the coverage radius set as $r_{\mathrm{cover}} = \sqrt{3 S_\mathrm{shape} /(2n\pi)}$. 
Then the metric is defined as $M_\mathrm{cover} = S_\mathrm{cover}/S_\mathrm{shape} \times 100\%$, where $S_\mathrm{shape}$ is the area of the desired shape and $S_\mathrm{cover}$ is the area covered by the union of all robot coverage regions within the shape. 
Full coverage yields $M_\mathrm{cover} = 100\%$ and no coverage results in $M_\mathrm{cover} = 0$. 

The fourth metric, convergence time, denoted as $T_\mathrm{conv}$, is defined as the time it takes for all robots to enter the desired shape. 
A shorter convergence time indicates higher method efficiency.

\subsection{Simulation Results}
\label{Subsec_simulation}

\begin{figure}[!t]  
	\centering
	\includegraphics[width=0.9\linewidth]{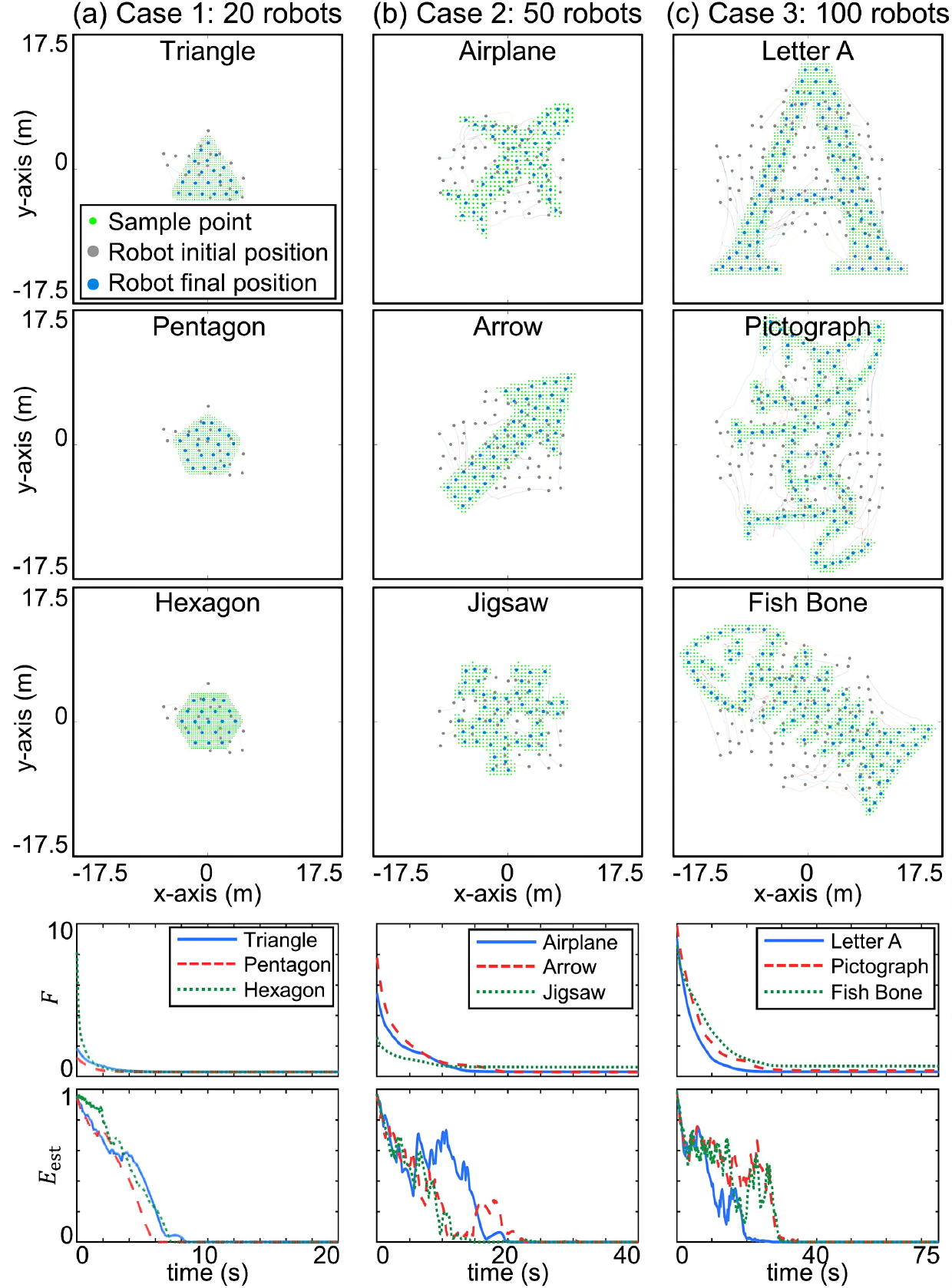}
	\caption{
		Simulations of our proposed strategy for different shapes and swarm sizes. 
		(a) Case 1: Trajectories and metrics of triangle, pentagon, and hexagon shapes formed by 20 robots. 
		(b) Case 2: Trajectories and metrics of airplane, arrow, and jigsaw shapes formed by 50 robots. 
		(c) Case 3: Trajectories and metrics of letter-A, pictograph, and fish-bone shapes formed by 100 robots. 
		\label{fig_coverage1}}
\end{figure}

Three simulation examples are presented to evaluate our proposed method. 
The simulation parameters are set as follows. 
The shape negotiation coefficients in \eqref{Equ_protocol_pos} and \eqref{Equ_protocol_ori} are set to $\alpha = 0.8$ and $c_1 = c_2 = 1.6$. 
The mass estimation coefficient in \eqref{Equ_massestimator} is set to $\gamma = 0.01$, and the meanshift controller coefficients in \eqref{Equ_mscommand_mod} and \eqref{Equ_CollAvdCmd} are set to $\sigma_1 = 30$, $\sigma_2 = 1000$, and $\varepsilon = 10^{-8}$. 
A suitable value of the bandwidth parameter $\beta$ can be determined by Algorithm \ref{Alg_DA} in the Appendix.
Specifically, in the following three examples, the bandwidth is set to $\beta = 1.5$.

The first example evaluates the proposed strategy across three swarm sizes ($n = 20, 50, 100$), with each size tested in three trials for distinct desired shapes (see Fig.~\ref{fig_coverage1}). 
In this example, the robot model parameters are set as follows: $r_\mathrm{sense} = 5$\,m, $r_\mathrm{avoid} = 1$\,m, and $v_\mathrm{max} = 1$\,m/s. 
In case 1, where the swarm size is $n = 20$ (see Fig.~\ref{fig_coverage1}(a)), the robot swarm smoothly forms three convex shapes: a triangle, pentagon, and hexagon. 
As shown in the performance metrics, the distribution-similarity error metric $F$ decreases rapidly and converges to a minimum, indicating successful fitting to the sample-point distribution. 
Concurrently, the mass estimation error $E_{\mathrm{est}}$ converges to zero within 10 seconds, validating the estimator. 
Cases 2 and 3 demonstrate that our strategy also assembles concave shapes (e.g., airplane, arrow, jigsaw) and more complex shapes (e.g., pictograph, fish-bone), even for large-scale swarms (see Fig.~\ref{fig_coverage1}(b) for $n = 50$, and Fig.~\ref{fig_coverage1}(c) for $n = 100$). 
The similarity metric and estimation error confirm stable performance across diverse shapes and swarm sizes. 

\begin{figure}[!t] 
	\centering
	\includegraphics[width=0.9\linewidth]{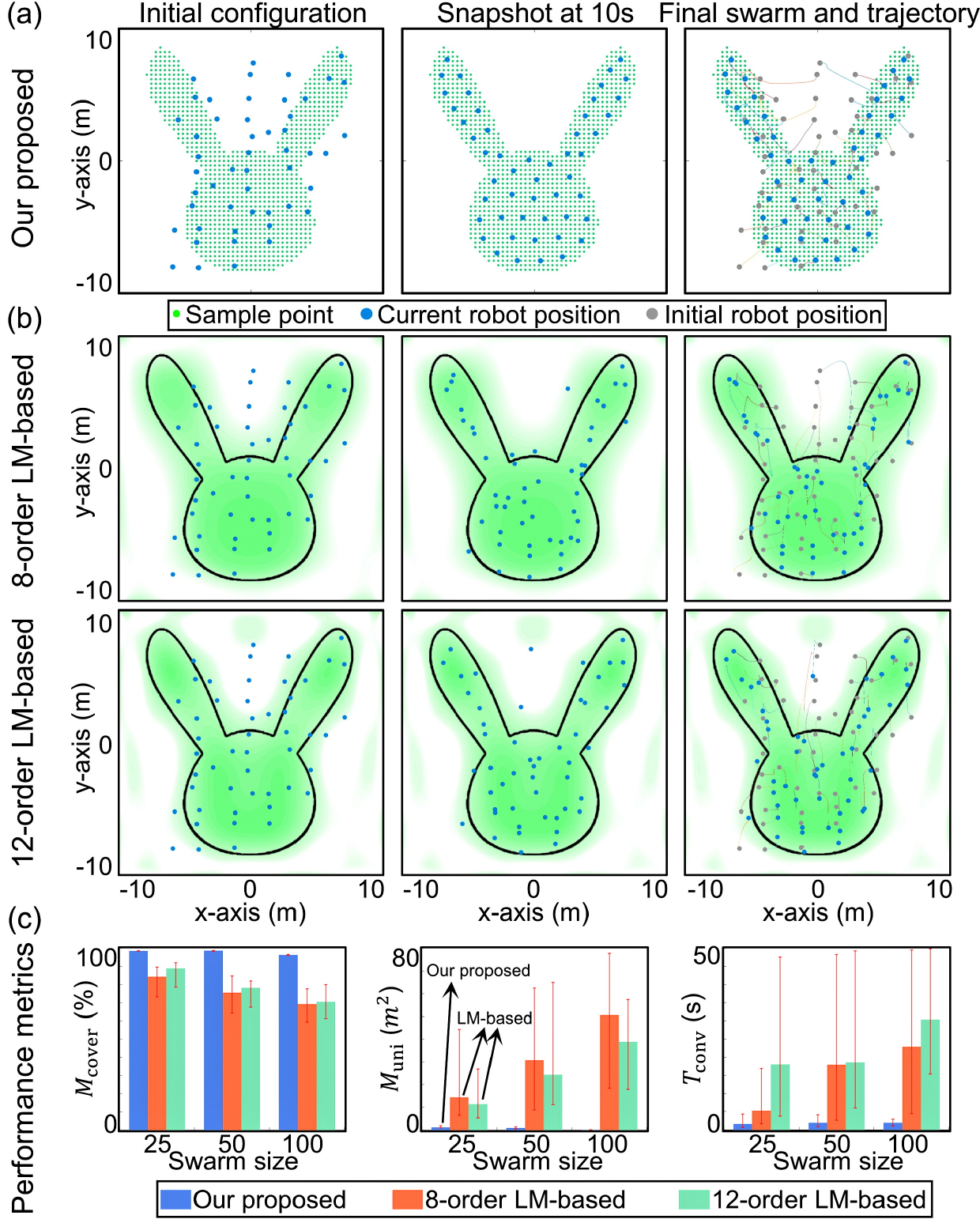}
	\caption{
		Comparative results of our strategy and the state-of-the-art \cite{Liu2024Self}, each achieving a bunny-head shape with 50 robots. 
		(a) Snapshots of our proposed strategy. 
		(b) Snapshots of strategy \cite{Liu2024Self} with 8th-order and 12th-order moment. 
		(c) Statistical results of the performance metrics, averaged over 20 trials of random initialization. 
		\label{fig_coverage2}}		
\end{figure}

In the second example, we compare our proposed strategy with the state-of-the-art Legendre moment (LM) method from \cite{Liu2024Self} (using 8th- and 12th-order moments), as illustrated in Fig.~\ref{fig_coverage2}. 
For a fair comparison, all swarms are tasked with forming an identical bunny-head shape, with deterministic knowledge of its location and orientation. 
Parameters are kept consistent: a sensing radius \( r_{\mathrm{sense}} = 10\,\mathrm{m} \) and a maximum velocity \( v_{\mathrm{max}} = 1\,\mathrm{m/s} \), while other parameters for our method match those in the first example.
Fig.~\ref{fig_coverage2}(a--b) presents snapshots and trajectories of 50-robot swarms using both strategies.
Although both achieve the desired formation from identical initial conditions, our strategy yields a shape that more closely follows the bunny-head contour, with fewer robots straying outside the boundary.
The performance is evaluated using three metrics: coverage rate \( M_{\mathrm{cover}} \), uniformity \( M_{\mathrm{uni}} \), and convergence time \( T_{\mathrm{conv}} \), as illustrated in Fig.~\ref{fig_coverage2}(c). 
For the LM-based methods---where some robots may never enter the shape---we define \( T_{\mathrm{conv}} \) as the time after which no additional robots enter, a lenient criterion that favors the LM strategy. 
Statistical results over 20 random trials confirm that our strategy outperforms the LM-based methods in all three metrics. 
This improvement stems from our mass distribution model, which reduces shape distortion compared to LM-based representations.

\begin{figure}[!t]  
	\centering
	\includegraphics[width=0.9\linewidth]{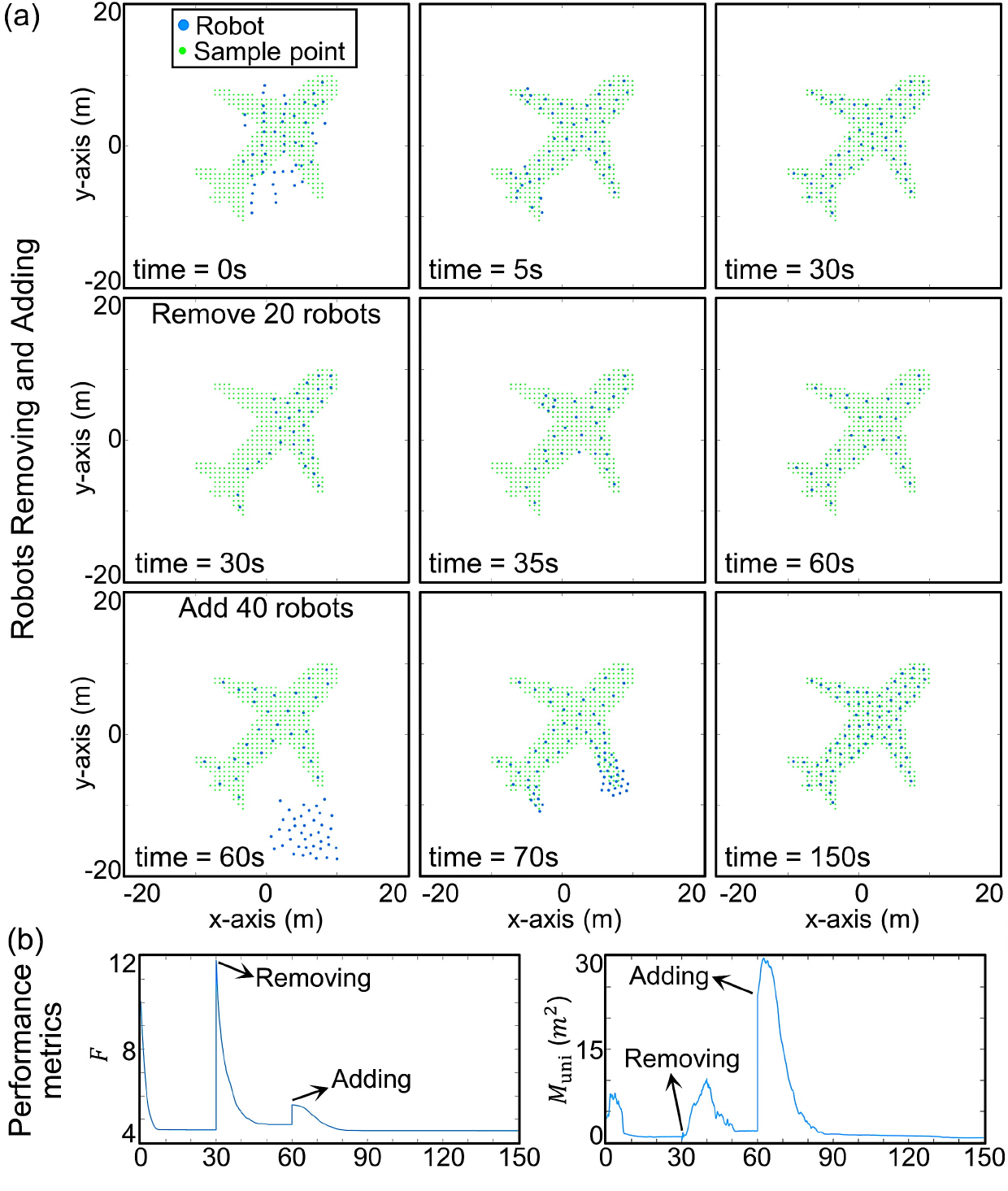}
	\caption{
		Simulation result for robots adding and removing. 
		(a) Snapshots of 50 robots forming an airplane shape after removing 20 and adding 40 robots. 
		(b) Corresponding performance metrics. 
		\label{fig_ExitEnter}}	
\end{figure}

The third example demonstrates the adaptability of our strategy to swarm-size variations, as illustrated in Fig.~\ref{fig_ExitEnter}. 
Simulation snapshots in Fig.~\ref{fig_ExitEnter}(a) show 50 robots initially positioned randomly, forming an airplane shape that stabilizes at $ t = 30 \,\text{s} $. 
The accomplishment of the formation is verified by the convergence of the similarity metric $ F $ and uniformity metric $ M_\mathrm{uni} $ in Fig.~\ref{fig_ExitEnter}(b). 
Subsequently, 20 robots are removed, and $ F $ is recalculated for the remaining 30 robots. 
To enable shape reformation, we reset the internal state of the mass estimator in \eqref{Equ_massestimator} to zero for all robots (i.e., $z_{k,i} = 0$ for all $i, k$), as the estimator requires an error-free initial condition $ \sum_{i=1}^n z_{k,i} = 0 $ \cite{Chen2012Distributed}. 
This reset can be implemented via decentralized broadcasting. 
After resetting, the 30 robots adjust to fill the shape, achieving the formation at $ t = 60 \,\text{s} $, with $ F $ and $ M_\mathrm{uni} $ converging to a new minimum. 
Then, 40 robots are added, and $ F $ is recalculated for the expanded swarm of 70 robots. 
The swarm proceeds to form the shape, finally achieving the airplane formation at $ t = 150 \,\text{s} $, as confirmed by metric convergence.

Notably, our method assumes a fixed desired shape, which may pose a scalability challenge if the swarm size $n$ grows beyond the shape's capacity. 
To ensure an effective formation, two conditions should be satisfied. 
First, to promote convergence to a near-optimal distribution and avoid local minima of the similarity metric $F$ (since Theorem~\ref{Thm_CtrlConvergence} only guarantees convergence to a critical point of $F$), the number of sample points $m$ should be sufficiently large. 
Based on our simulations, we recommend $m \geq 5n$.
Second, when collision avoidance is active, the inter-sample-point distance $d_{\text{pts}}$ must be chosen to accommodate the robots without overlap.
Modeling each robot as a circle with avoidance radius $r_{\text{avoid}}$, the shape's approximate area $S_{\text{shape}} \approx m d_{\text{pts}}^{2}$ must be large enough to contain $n$ such circles. 
This yields the condition $d_{\text{pts}} \geq \sqrt{\pi n / m} \cdot r_{\text{avoid}}$.

\subsection{Experiment Results}
\label{Subsec_experiment}

Experiments were conducted to validate the practical effectiveness of the proposed strategy. 
The experimental hardware setup, illustrated in Fig.~\ref{fig_Experiment}(a), consists of four primary components: 
(1) 10 TurtleBot3 Burger robots, which are ROS-2-based non-holonomic platforms; 
(2) a Nokov motion capture system that tracks each robot's position at 200\,Hz; 
(3) a workstation that receives these positions and computes velocity commands for each robot via MATLAB; 
and (4) a router that broadcasts these commands to all robots via Wi-Fi. 
Since the TurtleBot3 is underactuated, the velocity commands are transformed into linear and angular velocities according to \cite{Zhao2018General}. 
The MATLAB program also simulates the decentralized interactions defined in Section~\ref{Subsec_RobotModel}.

\begin{figure}[!t]
	\centering	\includegraphics[width=0.9\linewidth]{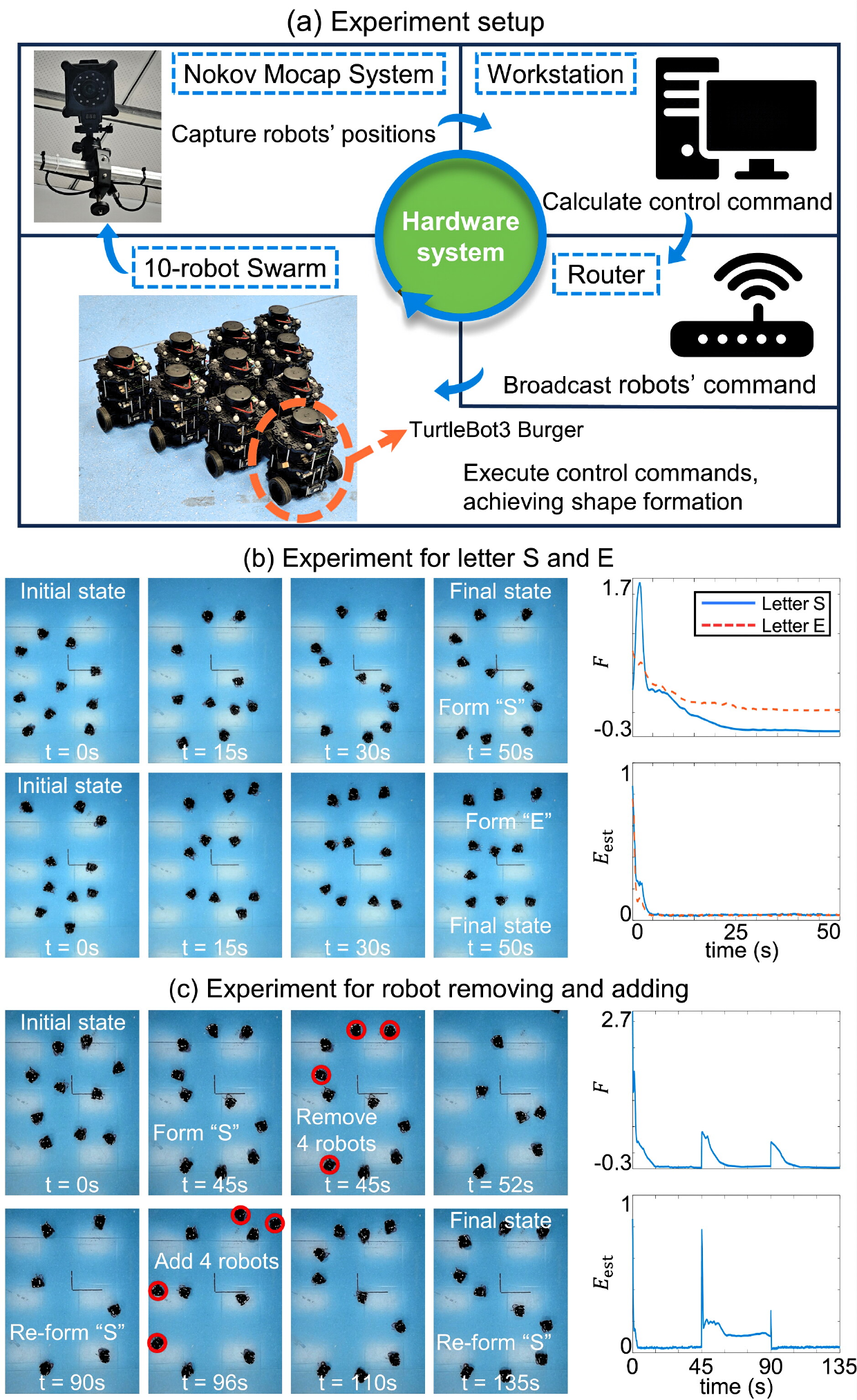}
	\caption{
		Experiment setup and results. 
		(a) Hardware setup of the experiments. 
		(b) Snapshots and corresponding metrics for 10 robots forming an “S” and an “E” shape.  
		(c) Snapshots and metrics demonstrating adaptability to swarm-size variations, where 4 robots are removed at $t=45\,\text{s}$ and reintroduced at $t=90\,\text{s}$ during an “S” shape formation.
		\label{fig_Experiment}}
\end{figure}

We first validate the proposed strategy by guiding a 10-robot swarm to form “S”- and “E”-shaped formations, as shown in Fig.~\ref{fig_Experiment}(b).  
The parameters are set as follows: sensing radius $r_\mathrm{sense} = 1.5$\,m, avoidance radius $r_\mathrm{avoid} = 0.35$\,m, gains $\sigma_1 = 2$, $\sigma_2 = 15$, and $\gamma = 0.05$.  
The bandwidth is $\beta = 5.5$ for the “S” shape and $\beta = 9$ for the “E” shape.  
The control and shape negotiation frequencies are both 10 Hz, while mass estimation operates at 100 Hz.  
As seen from the snapshots, the swarm successfully converges from a random initial configuration to the desired shapes within 50 seconds.  
The similarity metric may initially increase due to non-converged shape negotiation but decreases monotonically to a minimum after both negotiation and mass estimation converge (within 5 seconds), confirming successful formation.  

The second experiment demonstrates our strategy's adaptability to swarm-size variations (Fig.~\ref{fig_Experiment}(c)).  
Initially, the 10-robot swarm forms an “S” shape within 45 seconds, verified by the convergence of the similarity metric and estimation error.  
Subsequently, 4 robots are removed, and the mass estimator's internal state is reset to $z_{k,i} = 0$ (as in the adaptability simulation in Section~\ref{Subsec_simulation}).  
The remaining 6 robots reconfigure into the “S” shape by $t = 90$\,s.  
Then, the 4 robots are reintroduced, and the swarm of 10 reforms the shape by $t = 135$\,s.  
The decrease in the similarity metric and estimation error in each phase verifies both formation achievement and mass estimation convergence, validating practical adaptability to swarm-size variations.

\section{Conclusion}
\label{Sec_conclusion}

This paper has presented a distribution-based formation control method to achieve complex shape formation while adapting to swarm-size variations. 
First, we have introduced a mass distribution model defined over discrete sample points, which precisely represents the desired shape and decouples it from swarm size.
Second, we have developed a decentralized mass estimator based on decentralized average tracking, enabling each robot to accurately obtain the global mass distribution through local interactions.
Third, we have designed a meanshift controller that uses mass feedback to drive the robots toward the desired shape by fitting the sample-point distribution.
The stability of the controller has been established, and convergence has been ensured for convex shapes.
Finally, extensive simulations, experiments, and comparisons with state-of-the-art methods have demonstrated the effectiveness of the proposed method and its adaptability to varying swarm sizes.
Future work will focus on developing a robust mass estimator that eliminates the need for a broadcast signal upon robot removal, and on conducting real-world experiments with larger swarms.

\appendices
\section{Preliminary Lemmas}

To establish the theoretical results, we need to introduce the following necessary preliminary results. 

\begin{lemma}[LaSalle's Invariance Principle, Theorem 4.4 in \cite{Khalil2002Nonlinear}]
	\label{Lem_LasInvPrin}
	Denote $\Omega \subseteq \mathbb{R}^n$ as a compact set that is positively invariant with respect to the system $\dot{x} = f(x)$. 
	Here, $f: \mathbb{R}^n \rightarrow \mathbb{R}^n$ is locally Lipschitz. 
	Suppose that $V:  \mathbb{R}^n \to \mathbb{R}$ is a continuously differentiable function satisfying $\dot{V}(x) \leq 0$ in $\Omega$. 
	Let $E = \{ x \in \Omega : \dot{V}(x) = 0\}$ and $M$ be the largest invariant set in $E$. 
	Then every solution starting in $\Omega$ approaches $M$ as $t \to \infty$. 
\end{lemma}

\begin{lemma}
	$x e^{-\beta x^2} \leq 1/{\sqrt{2\beta e}}$ holds for all $x\ge 0$. 
	\label{Lem_Max_fi}
\end{lemma}
\begin{proof}
	Consider $f(x) = x e^{-\beta x^2}$ for $x \geq 0$. 
	Its derivative is $f'(x) = e^{-\beta x^2} - 2\beta x^2 e^{-\beta x^2} = (1 - 2\beta x^2) e^{-\beta x^2}$. 
	Since $e^{-\beta x^2} > 0$ for all \(x\), solutions to $f'(x) = 0$ satisfy $1 - 2\beta x^2 = 0$. 
	It follows that the critical point is $x^* = 1 /\sqrt{2\beta}$. 
	The derivative is positive when $x < \frac{1}{\sqrt{2\beta}}$ and negative when $x > \frac{1}{\sqrt{2\beta}}$, confirming $x^*$ as a maximum. 
	Thus, $f(x) \leq f(x^\ast) = 1 / \sqrt{2\beta e}$ holds for all $x \geq 0$.
\end{proof}

\begin{lemma}\label{Lem_E_ZeroMeasure}
	The set $E = \big\{\Lambda = [\xi^\top, \zeta^\top]^\top \in \mathbb{R}^{2d} : \xi,\zeta \in \mathbb{R}^d, \xi^\top \zeta = 0 \text{ or } \xi^\top \zeta = -(1-\varepsilon)\|\xi\|^2 \big\}$ has Lebesgue measure zero in $\mathbb{R}^{2d}$ for $d \geq 1$.
\end{lemma}

\begin{proof}
	Partition $E = A \cup B$ where $A = \big\{ \Lambda \in \R^{2d} : \xi^\top \zeta = 0 \big\} $ and $B = \big\{ \Lambda \in \R^{2d} : \xi^\top \zeta = -(1-\varepsilon)\|\xi\|^2 \big\}$. 
	Denote $\mu(\cdot)$ as the Lebesgue measure. 
	We then show $\mu(A) = \mu(B) = 0$. 
	
	We first show that $A$ has Lebesgue measure zero.
	Define $f: \R^{2d} \to \R$ by $f(\Lambda) = \xi^\top \zeta$. 
	Then $f$ is smooth and $A = f^{-1}(0)$. 
	Its gradient is $\nabla f = [\nabla_{\xi} f^\top , \nabla_{\zeta} f^\top]^\top = [\zeta^\top, \xi^\top]^\top$. 
	Note that critical points of $\nabla f$ satisfy $\nabla f = \textbf{0} \iff \xi = \zeta = \textbf{0} \iff \Lambda = \textbf{0}$. 	
	Then $df(v) = \nabla f^\top v$ is surjective from $\mathbb{R}^{2d}$ to $\mathbb{R}$ if and only if $\Lambda \neq \mathbf{0}$.	
	Thus $f$ is a submersion on $\R^{2d} \setminus \{\textbf{0}\}$. 
	By the Preimage Theorem \cite{Guillemin2010Differential}, $A = f^{-1}(0)$ is a submanifold of $\R^{2d} \setminus \{\textbf{0}\}$ with dimension $\mathrm{dim}A= 2d-1 < 2d$. 
	By Corollary 6.12 in \cite{Lee2012Manifolds}, we obtain $\mu(A) = 0$. 
	
	Next, we show that $B$ has Lebesgue measure zero. 
	Define $h: \R^{2d} \to \R$ by $ h(\Lambda) = \zeta^\top \xi + (1-\varepsilon) \|\xi\|^2$. 
	Then $B =  \{ \Lambda \in \mathbb{R}^{2d} : h(\Lambda) = 0 \} $.
	The gradient of $h$ is $\nabla h = [\nabla_{\xi} h^\top, \nabla_{\zeta} h^\top]^\top = [(\zeta + 2(1-\varepsilon)\xi)^\top, \xi^\top]^\top$. 
	Note that critical points of $\nabla h$ satisfy $\nabla h = \textbf{0} \iff \xi = \zeta = \textbf{0} \iff  \Lambda = \textbf{0}$. 
	Thus, similar to $A$, we can also obtain $\mu(B) = 0$. 
	
	By subadditivity of Lebesgue measure, $\mu(E) \leq \mu(A) + \mu(B) = 0 + 0 = 0$, i.e., $E$ has Lebesgue measure zero.
\end{proof}

\section{Proof of Theorems}
\label{Sec_ThmProof}
In this appendix, we present the proofs of Theorem~\ref{Thm_Nego_Consensus} and Theorem~\ref{Thm_Esti_Converge}.

\begin{proof}[Proof of Theorem~\ref{Thm_Nego_Consensus}]
	Let $\delta_i=q_{o,i}-\bar q_o$. 
	Consider the Lyapunov candidate $V(t)=\frac{1}{2}\sum \nolimits_{i=1}^n \delta_i(t)^\top \delta_i(t)$. 
	According to Theorem~1 in \cite{Wang2010Finite}, differentiating $V(t)$ with respect to $t$ gives 
	\begin{align*}
		\dot{V}(t) \le -\dfrac{c_1}{2}\left(4\lambda_2 \left(L(\mathcal{G}\left(t\right)\right)\right)^{\frac{1+\alpha}{2}} V(t)^{\frac{1+\alpha}{2}}
	\end{align*}
	where $c_1>0$ and $0<\alpha<1$ are two constants. 
	Here, $L(\mathcal{G}(t))$ is the Laplacian of $\mathcal{G}(t)$ and $\lambda_2(L(\mathcal{G}(t)))$ is the algebraic connectivity of $\mathcal{G}(t)$. 
	Since $\mathcal{G}(t)$ is connected for all $t \ge t_0$, $\lambda_2(L(\mathcal{G}(t)))>0$ always holds. 
	Note that all possible connections of $\mathcal{G}(t)$ are finite. 
	Thus, $K = \min \{\lambda_2 (L (\mathcal{G}(t))) : t \ge t_0 \}$ exists and is strictly positive. 
	Then we derive that
	\begin{align*}
		\dot{V}\left(t\right) \le -c_1 2^{\alpha} K^{\frac{1+\alpha}{2}} V\left(t\right)^{\frac{1+\alpha}{2}}. 
	\end{align*}
	By Theorem 1 in \cite{Bhat2000Finite}, there exists 
	\begin{align*}
		t^{\ast} \le \frac{ 2^{ 1 - \alpha} V(t_0)^{\frac{1-\alpha}{2}}}{ c_1  K^{\frac{1+\alpha}{2}} (1-\alpha)} 
	\end{align*}
	such that $V(t)=0$ for all $t\ge t^*$, which implies that $q_{o,i}(t) = \bar q_o$ for all $t\ge t^*$. 
	Similarly, following the proof of protocol \eqref{Equ_protocol_pos}, $\theta_{o,i}(t) = \bar \theta_o$ can also be reached in finite time.
\end{proof}

\begin{proof}[Proof of Theorem~\ref{Thm_Esti_Converge}]
	For fixed $k\in \{1,\ldots,m\}$, define $r_i = e^{-\beta \|p_i - q_k\|^2}$ and $f_i = \dot{r}_i$ for $i\in\{1, \ldots, n\}$. 
	Differentiating $\hat{P}_{k,i}$ with respect to $t$ and substituting \eqref{Equ_massestimator} into it yields 
	\begin{align}
		\dot{\hat{P}}_k^i = f_i + \gamma \sum_{j \in \mathcal{N}_i} \operatorname{sign}\left(\hat{P}_{k,j} - \hat{P}_{k,i} \right).
		\label{Equ_massestimator_diff}
	\end{align}
	According to Theorem 3 in \cite{Chen2012Distributed}, if there exists a constant $\bar f>0$ such that $|f_i| \leq \bar{f}$ and $\gamma > (n-1)\bar{f}$ for any $i\in \{1,\cdots,n\}$, then by the action of \eqref{Equ_massestimator_diff}, $\lim_{t \rightarrow \infty }\hat{P}_{k,i}(t) = P_k$ holds for all $i$. 
	From the definition of $f_i$, it follows that 
	\begin{align}
		\begin{split}
			\left|f_i\right| &= \left| \frac{d}{dt} e^{-\beta \|p_i - q_k\|^2} \right| \\
			&= 2\beta e^{-\beta \|p_i - q_k\|^2} \left| (q_k - p_i)^\top v_i \right| \\
			&\leq 2\beta e^{-\beta \|p_i - q_k\|^2} \|q_k - p_i\| \|v_i\| \\
			&\leq 2\beta \frac{\|q_k - p_i\|}{e^{\beta \|p_i - q_k\|^2}} v_{\mathrm{max}}.
		\end{split}
		\label{Equ_estimatorbound}
	\end{align}
	By Lemma~\ref{Lem_Max_fi}, \eqref{Equ_estimatorbound} implies that $|f_i|\leq 2\beta v_{\mathrm{max}}/{\sqrt{2\beta e}}$, and consequently $|f_i| \leq \bar{f}:= \sqrt{{2\beta}/{e}} v_{\mathrm{max}}$. 
	Since $\gamma > (n-1)\sqrt{2\beta/e} \, v_{\mathrm{max}}$, then $\gamma > (n-1)\bar{f}$. 
	As a result, it holds that $\lim_{t \rightarrow \infty }\hat{P}_{k,i}(t) = P_k$ for $i\in \{1,\cdots,n\}$ and $k \in \{1,\ldots,m\}$.     
\end{proof}

\section{Proof of Lemmas}
\label{Sec_Lemma12Proof}

In this appendix, we present the theoretical analyses of all theorems in this paper.

\begin{proof}[Proof of Lemma~\ref{Lem_uniformity}]
	According to the Cauchy-Schwartz inequality, it follows that $|\mathbf{P}^\top \mathbf{1}_m| \le \| \mathbf{P}\| \| \mathbf{1}_m\|$. Here, $\mathbf{1}_m = [1,\ldots,1]^\top \in \mathbb{R}^m$ and $ \| \mathbf{P}\| = \sqrt{\sum_{k=1}^m P_k^2}$ is the Euclidean norm. 
	The equality holds iff $\mathbf{P} = c \mathbf{1}_m$ for some $c \in \mathbb{R}$, which is equivalent to $P_k = P_l$ for all $k,l \in \{1,\ldots,m \}$. 
	It further derives that $-\ln ( | \mathbf{P}^\top \mathbf{1}_m  | / (\| \mathbf{P}\| \| \mathbf{1}_m \|)) \ge 0$ and the equality holds iff $P_k = P_l$ for all $k,l$. 
	Then, we have 
	\begin{align*}
		\begin{split}
			0 &\le -\ln \left( \frac{\left|\mathbf{P}^\top  \mathbf{1}_m \right|}{\left\| \mathbf{P} \| \right\| \mathbf{1}_m \|}\right)  \\
			&= -\ln \left( \sum_{k=1}^m P_k \right)  + \ln \left\| \mathbf{P} \right\| + \dfrac{1}{2}\ln m \\
			&\le -\dfrac{1}{m} \sum_{k=1}^m \ln P_k  -(1-\dfrac{1}{2} )\ln m + \ln \| \mathbf{P} \|  \\
			& = F_{\mathrm{uni}}
		\end{split}		
	\end{align*} 
	where the second inequality is derived by Jensen's inequality, of which the equality holds iff $P_k = P_l$ for all $k,l$. 
	Therefore, $F_{\mathrm{uni}} \ge 0$ and the equality holds iff $P_k = P_l$ for all $k,l \in \{1,\ldots,m \}$.
\end{proof}

\begin{proof}[Proof of Lemma~\ref{Lem_Smooth}]
	First, we establish the smoothness of the mass function $P_k(\mathbf{p})$ with respect to $\mathbf{p}$. 
	Recall that $P_k : \mathbb{R}^{dn} \rightarrow \mathbb{R}_{> 0}$ is defined as $P_k(\mathbf{p}) = \frac{1}{n} \sum_{i=1}^n e^{-\beta\|p_i-q_k\|^2}$. 
	Define the squared distance function $d_{ik}: \mathbb{R}^{dn} \rightarrow \mathbb{R}$ as $d_{ik}(\mathbf{p}) := \|p_i-q_k\|^2$. 
	This function is smooth in $\mathbf{p}$ over $\mathbb{R}^{dn}$. 
	Let $\tau: \mathbb{R} \rightarrow \mathbb{R}_{> 0}$ denote the exponential function $\tau(x) = e^{-\beta x}$, which is smooth. 
	Since $P_k$ is a finite sum of compositions of smooth functions (i.e., $P_k = \frac{1}{n} \sum_{i=1}^n \tau \circ d_{ik}$), then $P_k$ is smooth in $\mathbf{p}$.
	
	Next, we show the smoothness of $F$. 
	Since $P_k > 0$ holds strictly and the logarithm function $\ln(\cdot): \mathbb{R}_{>0} \rightarrow \mathbb{R}$ is smooth, each $\ln P_k$ is smooth in $\mathbf{p}$. 
	Thus, as a finite sum of smooth functions and a constant, the metric $F = -\frac{1}{m}\sum_{k=1}^m \ln P_k - \frac{1}{2} \ln m$ is smooth with respect to $\mathbf{p}$. 
\end{proof}



\begin{proof}[Proof of Lemma~\ref{Lem_ConflictFree}]
	Reorganizing \eqref{Equ_ConflictFree} gives 
	\begin{align}\label{Equ_ConflictFree2}
		(v_i^\mathrm{cv})^\top v_i^\mathrm{ms}  \ge -(1-\varepsilon) \|v_i^\mathrm{ms}\|^2.
	\end{align}
	Then we only need to show that \eqref{Equ_ConflictFree2} holds. 
	If $(\tilde{v}_i^\mathrm{cv})^\top v_i^\mathrm{ms} \ge 0$, then we have $(v_i^\mathrm{cv})^\top v_i^\mathrm{ms} \ge 0 \ge -(1-\varepsilon) \|v_i^\mathrm{ms}\|^2$, which implies that \eqref{Equ_ConflictFree2} holds. 
	If $(\tilde{v}_i^\mathrm{cv})^\top v_i^\mathrm{ms} < 0$, it follows that 
	\begin{align*}
		\begin{split}         
			\left(v_i^\mathrm{cv}\right)^\top v_i^\mathrm{ms} 
			= & \varphi \cdot \min \left\{ -\frac{\left(1-\varepsilon \right) \left\| v_i^\mathrm{ms} \right\|^2}{\left( v_i^\mathrm{ms} \right)^\top \tilde{v}_i^\mathrm{cv}}, 1 \right\} \cdot \left(\tilde{v}_i^\mathrm{cv} \right)^\top   v_i^\mathrm{ms} \\
			\ge & \left(-\dfrac{ \left(1-\varepsilon \right) \left\| v_i^\mathrm{ms} \right\|^2}{\left( \tilde{v}_i^\mathrm{cv} \right)^\top v_i^\mathrm{ms}}\right) \cdot \left(\tilde{v}_i^\mathrm{cv}\right)^\top v_i^\mathrm{ms} \\
			=& -\left( 1-\varepsilon \right) \left\|v_i^\mathrm{ms} \right\|^2.
		\end{split}
	\end{align*}
	Thus, \eqref{Equ_ConflictFree2} always holds, and so does \eqref{Equ_ConflictFree}. 
\end{proof}    

\begin{proof}[Proof of Lemma~\ref{Lem_p_bounded}]
	We show that $\mathbf{p}$ is bounded by showing the boundedness of $p_i$ for all $i\in\{1,\ldots,n\}$. 
	The core idea is to construct a constant $L > 0$ such that $\frac{d}{dt} \|p_i\| < 0$ when $\|p_i\| > L$, ensuring that $p_i$ remains bounded. 
	To do that, we analyze the sign of $\frac{d}{dt}\|p_i\| = p_i^\top v_i / \|p_i\|$. 
	Substituting $v_i$ as defined in \eqref{Equ_CtrlLaw_mod}, for $\|p_i\| > L$, we obtain 
	\begin{align}
		\begin{split}       
			\dfrac{d}{dt} \left\|p_i\right\| 
			&= \Gamma\left(\| v_i^\mathrm{ms} + v_i^\mathrm{cv}\|\right) \dfrac{p_i^\top \left(v_i^\mathrm{ms} + v_i^\mathrm{cv} \right)}{\|p_i\|} \\
			&\leq \Gamma\left(\| v_i^\mathrm{ms} + v_i^\mathrm{cv}\|\right) \left( \left\| v_i^\mathrm{cv}\right\| + \dfrac{p_i^\top v_i^\mathrm{ms}}{\left\| p_i \right\|} \right)
		\end{split}
		\label{Equ_DotNormPi}
	\end{align}    
	where $\Gamma > 0$ by \eqref{Equ_CtrlLaw_mod}. 
	Then, we examine the sign of the term inside the parentheses.
	
	First, we show that $\|v_i^\mathrm{cv}\|$ is upper-bounded. 
	According to the definition of $\kappa_2$ in \eqref{Equ_CollAvdCmd}, we have $\kappa_2 \in [0, 1]$, which implies 
	\begin{align*}
		\left\| v_i^\mathrm{cv} \right\| \leq \sigma_2 \sum_{j \in \mathcal{N}_i^\mathrm{cv}} (r_{\mathrm{avoid}} - \|p_i - p_j\|) \leq \sigma_2 (n - 1) r_{\mathrm{avoid}}.
	\end{align*}
	
	Second, we analyze the upper-bound of $p_i^\top v_i^\mathrm{ms} / \| p_i \|$. 
	Let $R = \max_k \|q_k\|$ and $B(R) = \{ x \in \mathbb{R}^d: \|x\| \leq R \}$ be a compact ball. 
	Substituting $\alpha_k = (P_k)^{-1} \omega (\|q_k - p_i\|) > 0$ and $y = \sum \alpha_k q_k / (\sum \alpha_k)$ into \eqref{Equ_mscommand_mod}, we obtain $v_i^\mathrm{ms} = \sigma_1(y - p_i) / m$. 
	The definition of $y$ implies $\|y\| < R$, i.e., $y \in B(R)$. 
	This indicates that $v_i^\mathrm{ms}$ always points into $B(R)$. 
	For $\|p_i\| > R$ (i.e., $p_i \notin B(R)$), the angle between $v_i^\mathrm{ms}$ and $p_i$ is obtuse, indicating $p_i^\top v_i^\mathrm{ms} < 0$. 
	Let $x$ be the orthogonal projection of the origin onto the line of $v_i^\mathrm{ms}$, then $x \in B(R)$. 
	From the similar triangle relationship shown in Fig.~\ref{Fig_boundedness}, we have 
	\begin{align*}
		\dfrac{\left\|p_i\right\|}{\left\|p_i - x\right\|} = \dfrac{\left\|v_i^\mathrm{ms}\right\|}{\left| p_i^\top v_i^\mathrm{ms}/\left\|p_i\right\| \right|}.
	\end{align*}
	Then for $\|p_i\| > 2R$, it follows that 
	\begin{align*}
		\left| \dfrac{p_i^\top v_i^\mathrm{ms}}{\|p_i\|} \right| &= \frac{\|v_i^\mathrm{ms}\| \|p_i - x\|}{\|p_i\|} \\
		&\geq \dfrac{\sigma_1}{m} \dfrac{\| y - p_i \| \|x - p_i\|}{\|p_i\|} \\
		&\geq \dfrac{\sigma_1}{m} \left| \|p_i\| - (\|x\| + \|y\|) + \frac{\|x\|\|y\|}{\|p_i\|} \right| \\
		&> \dfrac{\sigma_1}{m} (\|p_i\| - 2R).
	\end{align*}
	Combining this with $p_i^\top v_i^\mathrm{ms} < 0$ yields $p_i^\top v_i^\mathrm{ms}/\|p_i\| < -\sigma_1 (\|p_i\| - 2R)/m$. 
	
	\begin{figure}[!t]
		\centering
		\includegraphics[width=1\linewidth]{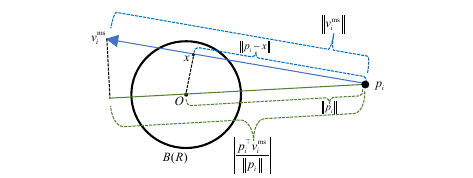}
		\caption{Relationship between $p_i$, $v_i^\mathrm{ms}$, $x$, and $|p_i^\top v_i^\mathrm{ms} /\|p_i\||$. Note that $|p_i^\top v_i^\mathrm{ms} /\|p_i\||$ is the norm of the projection of $v_i^\mathrm{ms}$ onto $p_i$. \label{Fig_boundedness}}
	\end{figure}  
	
	Next, we demonstrate the boundedness of $p_i$.  
	Taking $L = 2R + \sigma_2 m (n-1) r_\mathrm{avoid} /\sigma_1$ as the desired constant and substituting the upper bounds of $\|v_i^\mathrm{cv}\|$ and $\|v_i^\mathrm{ms} \|$ into \eqref{Equ_DotNormPi} yields
	\begin{align*}
		\frac{d}{dt} \|p_i\| &< -\Gamma \left[ \dfrac{\sigma_1}{m} (\|p_i\| - 2R) - \sigma_2 (n-1) r_\mathrm{avoid} \right] < 0
	\end{align*}
	for $\|p_i\| > L$. 
	We now prove $\|p_i\| \le M = \max \{ L, \|p_1(t_0)\|, \ldots, \|p_n(t_0)\| \}$ by contradiction. 
	Suppose that $\| p_i(t_1) \| > M$ for some $t_1 > t_0$. 
	Since $p_i(t)$ is continuous and $\|p_i(t_0)\| \leq M$, there exists $t_2 \in [t_0, t_1)$ such that $\|p_i(t_2)\| = M$ and $\|p_i(t)\| > M$ for all $t \in (t_2, t_1]$. 
	Since $\frac{d}{dt} \|p_i(t)\| < 0$ whenever $\|p_i(t)\| > L$, the norm $\|p_i(t)\|$ is strictly decreasing on $[t_2, t_1]$. 
	This implies $\|p_i(t_1)\| < \|p_i(t_2)\| = M$, which contradicts $\|p_i(t_1)\| > M$. 
	Thus, $\|p_i(t)\| \leq M$ for all $t \geq t_0$, proving that $p_i$ is bounded. 
	Since the swarm size $n$ is finite, $\|\mathbf{p}\| = \sqrt{\sum_{i=1}^n \|p_i\|^2} \le \sqrt{n} M$ holds, which implies that $\mathbf{p}$ is also bounded. 
\end{proof}

\begin{proof}[Proof of Lemma~\ref{Lem_MS_Lipschitz}]
	Substituting the specific definition of $\kappa_1$ into \eqref{Equ_mscommand} and replacing $\hat{P}_{k,i}$ with $P_k$ by Assumption~\ref{Asm_Pki=Pk}, it follows that 
	\begin{align*}
		v_i^\mathrm{ms}(\mathbf{p}) = -\frac{\sigma_1 \nabla_{p_i} F(\mathbf{p})}{2\beta \sum_{k=1}^m \left(P_k(\mathbf{p})\right)^{-1} e^{-\beta \left\|p_i - q_k\right\|^2}}.
	\end{align*}
	By Lemma~\ref{Lem_Smooth}, $F(\mathbf{p})$ is smooth in $\mathbf{p}$, thus $\nabla_{p_i} F(\mathbf{p})$ is smooth in $\mathbf{p}$. 
	Since $P_k(\mathbf{p}) > 0$ is smooth in $\mathbf{p}$, it follows that $v_i^\mathrm{ms}$ is smooth in $\mathbf{p}$. 
	Moreover, as a smooth function must be locally Lipschitz, $v_i^\mathrm{ms}$ is locally Lipschitz in $\mathbf{p} \in \mathbb{R}^{dn}$.
\end{proof}

\begin{proof}[Proof of Lemma~\ref{Lem_CV_Lipschitz}]
	Denote $\eta = \eta_1 \cdot \eta_2$ with $\eta_1 = \varphi = \min \{ \|v_i^\mathrm{ms} \|^2 / \varepsilon, 1 \}$ and 
	\begin{align*}
		\eta_2 =  \begin{cases}
			1, & \left( v_i^\mathrm{ms} \right)^\top \tilde{v}_i^\mathrm{cv} \ge 0 \\
			\min \left\{ -\dfrac{(1-\varepsilon) \| v_i^\mathrm{ms} \|^2}{\left( v_i^\mathrm{ms} \right)^\top \tilde{v}_i^\mathrm{cv}}, 1\right\}, & \left( v_i^\mathrm{ms} \right)^\top \tilde{v}_i^\mathrm{cv} < 0.
		\end{cases}
	\end{align*}
	Then, $\eta = \kappa_2$ and $v_i^\mathrm{cv}$ can be rewritten as $v_i^\mathrm{cv} = \eta \cdot \tilde{v}_i^\mathrm{cv}$. 
	Since the multiplication of locally Lipschitz functions is locally Lipschitz, we show the locally Lipschitz continuity of $v_i^\mathrm{cv}$ by proving that $\eta$ and $\tilde{v}_i^\mathrm{cv}$ are locally Lipschitz in $\mathbf{p}$, respectively. 
	
	First, we show that $\tilde{v}_i^\mathrm{cv}$ is locally Lipschitz in $\mathbf{p}$. 
	In order to do that, we reorganize it as $\tilde{v}_i^\mathrm{cv}(\mathbf{p}) = \sigma_2 \sum_{j\ne i} \Phi(e_{ij}(\mathbf{p}))$, where $\Phi(x) = \phi(\|x\|)x$ for $x\in \mathbb{R}^d$, $e_{ij}(\mathbf{p}) = p_i - p_j$ for $\mathbf{p} \in \mathbb{R}^{dn}$ and $\phi(z) = \max \left\{ (r_\mathrm{avoid} - z)/(z + \varepsilon) , 0 \right\}$ for $z \ge 0$. 
	The function $e_{ij}$ is locally Lipschitz in $\mathbf{p}$ since it is linear in $\mathbf{p}$. 
	We then show that $\Phi(x)$ is locally Lipschitz in $x \in \mathbb{R}^d$. 
	For this purpose, we first demonstrate that $\phi(z)$ is globally Lipschitz. 
	The derivative of $(r_\mathrm{avoid} - z)/(z + \varepsilon)$ with respect to $z$ is $- (r_{\mathrm{avoid}} + \varepsilon)/ (z + \varepsilon)^2$, whose absolute value is bounded by $L_{\phi} = (r_{\mathrm{avoid}} + \varepsilon)/\varepsilon^2$. Thus, $(r_\mathrm{avoid} - z)/(z + \varepsilon)$ is $L_{\phi}$-Lipschitz. Since $\max\{ \cdot , 0\}$ is $1$-Lipschitz, we conclude that $\phi(z)$ is $L_{\phi}$-Lipschitz in $z$. 
	Then, for $x_0 \in \mathbb{R}^d$ and for $x_1, x_2 \in B(x_0,\varepsilon) = \{x \in \mathbb{R}^d : \|x-x_0\| < \varepsilon \}$, we have 
	\begin{align*}
		& \|\Phi(x_1) - \Phi(x_2)\| \\
		\le& \left|\phi(\|x_1\|) - \phi(\|x_2\|) \right| \|x_1\| + \phi(\|x_2\|) \|x_1 - x_2\| \\
		\le& L_{\phi} \|x_1 - x_2\| \|x_1\| + \dfrac{r_\mathrm{avoid}}{\varepsilon} \|x_1 - x_2\| \\
		\le& \left(L_{\phi} \left(\|x_0\| + \varepsilon \right) + \dfrac{r_\mathrm{avoid}}{\varepsilon} \right) \|x_1 - x_2\|.
	\end{align*}
	As can be seen, $\Phi(x)$ is locally Lipschitz on $\mathbb{R}^d$. 
	Thus, as a summation and composition of $\Phi(x)$ and $e_{ij}(\mathbf{p})$, $\widetilde{v}_i^\mathrm{cv}(\mathbf{p})$ is locally Lipschitz in $\mathbf{p}$. 
	
	Second, we show that $\eta$ is locally Lipschitz in $\mathbf{p}$. 
	Note that $\eta$ is only related to $v_i^\mathrm{ms}$ and $\tilde{v}_i^\mathrm{cv}$. 
	Since the locally Lipschitz of $v_i^\mathrm{ms}$ (Lemma~\ref{Lem_MS_Lipschitz}) and $\tilde{v}_i^\mathrm{cv}$ in $\mathbf{p}$ has been proved, we need only show that $\eta$ is locally Lipschitz in $[(v_i^\mathrm{ms})^\top, (\tilde{v}_i^\mathrm{cv})^\top]^\top$. 
	Denote $\xi = v_i^\mathrm{ms}$, $\zeta = \widetilde{v}_i^\mathrm{cv}$, and $\Lambda = [\xi^\top, \zeta^\top]^\top$. 
	Then, $\eta$ is reformulated as $\eta(\Lambda) = \eta_1(\Lambda) \eta_2(\Lambda)$ with $\eta_1(\Lambda) = \min \{ \|\xi\|^2 / \varepsilon, 1\}$ and 
	\begin{align*}
		\eta_2(\Lambda) =  \begin{cases}
			1, &  \zeta^\top \xi \ge 0 \\
			\min \left\{ -\dfrac{(1-\varepsilon) \|\xi\|^2}{\zeta^\top \xi}, 1\right\}, &\zeta^\top \xi < 0
		\end{cases}.
	\end{align*}	
	The objective now is to prove that $\eta(\Lambda)$ is locally Lipschitz in $\Lambda \in \mathbb{R}^{2d}$. 
	To do that, we analyze its locally Lipschitz continuity in the following two distinguished subsets. 
	
	$\bullet$ Scenario 1:  $\{ \Lambda \in \mathbb{R}^{2d}: \|\xi\| > 0\}$.
	
	We first show that $\eta$ is locally Lipschitz on $\{ \Lambda \in \mathbb{R}^{2d}: \|\xi\| > 0\}$ by proving the locally Lipschitz continuity of $\eta_1$ and $\eta_2$ separately. 
	The function $\eta_1$ is locally Lipschitz in $\mathbf{p}$ since $\|\xi\|^2/\varepsilon$ has a bounded gradient on bounded sets and $\min\{\cdot, 1\}$ is $1$-Lipschitz. 
	Then, we discuss the locally Lipschitz of $\eta_2$ at any point $\Lambda_0 = [\xi_0^\top, \zeta_0^\top]^\top$ in $\{ \Lambda \in \mathbb{R}^{2d}: \|\xi\| > 0\}$. 
	We consider the following three cases. 
	
	Case 1: $\zeta_0^\top \xi_0 > 0$.     
	In this case, the continuity of $\zeta^\top \xi$ ensures the existence of a neighborhood $U$ of $\Lambda_0$ such that $\zeta^\top \xi > 0$ for all $\Lambda \in U$. 
	Then $\eta_2(\Lambda) = 1$ on $U$, which implies that $\eta_2$ is locally Lipschitz at $\Lambda_0$.
	
	Case 2: $\zeta_0^\top \xi_0 < 0$.     
	In this case, there exists a bounded neighborhood $U$ of $\Lambda_0$ such that $\zeta^\top \xi < \zeta_0^\top \xi_0 / 2 < 0$ for all $\Lambda \in U$. 
	Define $g(\Lambda) = -(1-\varepsilon)\|\xi\|^2 / (\zeta^\top \xi)$. 
	The gradient of $g$ is given by 
	\begin{align*}
		\nabla_\Lambda g\left(\Lambda\right) = \dfrac{\left(1 - \varepsilon \right)}{\left( \zeta^\top \xi\right)^2}  
		\begin{bmatrix} 
			-2(\zeta^\top \xi)\xi + \|\xi\|^2 \zeta \\ 
			\|\xi\|^2 \xi 
		\end{bmatrix},
	\end{align*}
	which is bounded since both $\|\xi\|$ and $\|\zeta\|$ are upper-bounded in $U$ and $1/(\zeta^\top \xi)^2 < 4/ (\zeta_0^\top \xi_0)^2$. 
	Thus, $g$ is Lipschitz on $U$. 
	Since $\eta_2(\Lambda) = \min \{g(\Lambda), 1\}$ and $\min\{\cdot,1\}$ is $1$-Lipschitz, it follows that $\eta_2$ is Lipschitz on $U$. 
	Therefore, $\eta_2$ is locally Lipschitz at $\Lambda_0$. 
	
	Case 3: $\zeta_0^\top \xi_0 = 0$.     
	Denote $\delta_1 = \|\xi_0\| / 4$. 
	By the triangle inequality, $\|\xi\| > \|\xi_0\| / 2$ for all $\xi$ with $\| \xi - \xi_0\| < \delta_1$. 
	Let $c = (1-\varepsilon) \left(\|\xi_0\| / 2\right)^2$. 
	There exists $\delta_2 > 0$ such that $| \zeta^\top \xi | < c$ for all $\Lambda$ with $\| \Lambda - \Lambda_0\| < \delta_2$ by the continuity of $|\zeta^\top \xi|$. 
	Let $\delta = \min \{ \delta_1, \delta_2\}$ and take $\Lambda \in B(\Lambda_0, \delta)$. 
	If $\zeta^\top \xi \ge 0$, then $\eta_2(\Lambda) = 1$. 
	If $\zeta^\top \xi < 0$, then $| \zeta^\top \xi| < c < (1-\varepsilon) \|\xi\| ^2$, and hence $g(\Lambda) = (1-\varepsilon) \|\xi\| ^2 / | \zeta^\top \xi| > 1$, which implies $\eta_2(\Lambda) = \min \{g(\Lambda),1 \} = 1$. 
	Thus, $\eta_2(\Lambda) = 1$ for all $\Lambda \in B(\Lambda_0, \delta)$, proving that $\eta_2$ is locally Lipschitz at $\Lambda_0$. 
	
	To sum up, $\eta_2$ is locally Lipschitz on $\{\Lambda \in \mathbb{R}^{2d}: \|\xi\| > 0\}$, and $\eta = \eta_1 \cdot \eta_2$ is locally Lipschitz on this set.     
	
	$\bullet$ Scenario 2: $ \{ \Lambda \in \mathbb{R}^{2d}: \|\xi\| = 0\}$. 
	
	Next, we show that $\eta$ is locally Lipschitz on $\{\Lambda \in \mathbb{R}^{2d}: \|\xi\| = 0\}$. 
	For any $\Lambda_0 = [\mathbf{0}^\top, \zeta_0^\top]^\top $ in $\{\Lambda \in \mathbb{R}^{2d}: \|\xi\| = 0\}$, let $U = \{\Lambda : \|\Lambda - \Lambda_0\| < \sqrt{\varepsilon}/2\}$.         
	Next, we take two steps to demonstrate the Lipschitz continuity of $\eta$ in $U$. 
	
	The first step is to show that the gradient of $\eta$ is continuous and bounded almost everywhere in $U$. 
	By the definition of $U$, we have $\|\xi\| < \sqrt{\varepsilon}/2$ for $\Lambda \in U$, which implies $\eta_1 = \|\xi\|^2/\varepsilon$. 
	Then $\eta_1$ is differentiable in $U$, and the differentiability of $\eta$ is equivalent to the differentiability of $\eta_2$. 
	From the definition of $\eta_2$, it is non-differentiable only when $\xi^\top \zeta = 0$ or $\xi^\top \zeta = -(1-\varepsilon)\|\xi\|^2 $. 
	Thus $\eta(\Lambda)$ is differentiable in $U \setminus E$, where $E$ is the non-differentiable set defined as $E = \{ \Lambda \in U : \xi^\top \zeta = 0 \text{ or } \xi^\top \zeta = -(1-\varepsilon)\|\xi\|^2 \}$, which has Lebesgue measure zero (proved by Lemma~\ref{Lem_E_ZeroMeasure}). 
	Then, for $\Lambda \in U \setminus E$, there are three cases. 
	
	Case 1: $\xi^\top \zeta > 0$. If $\xi^\top \zeta > 0$, then $\eta(\Lambda) = \|\xi\|^2/\varepsilon$, which implies that $\nabla_\Lambda \eta = [2\xi^\top / \varepsilon, 0 ]^\top$ is continuous in $\Lambda$ and is bounded by $\|\nabla_\Lambda \eta\| < \varepsilon^{-1/2}$. 
	
	Case 2: $\xi^\top \zeta < 0$ and $g(\Lambda) > 1$. If $\xi^\top \zeta < 0$ and $g(\Lambda) > 1$, then $\eta(\Lambda) = \|\xi\|^2/\varepsilon$.
	Same as case 1, we also obtain that $\nabla_\Lambda \eta$ is continuous and bounded by $\|\nabla_\Lambda \eta\| < \varepsilon^{-1/2}$. 
	
	Case 3: $\xi^\top \zeta < 0$ and $g(\Lambda) < 1$. If $\xi^\top \zeta < 0$ and $g(\Lambda) < 1$, then $\eta = -(1-\varepsilon)\|\xi\|^4 / (\varepsilon \xi^\top \zeta)$.     
	It follows that $\nabla_\Lambda \eta = [\nabla_{\xi} \eta^\top, \nabla_{\zeta} \eta^\top]^\top$ where  
	\begin{align*}
		\nabla_{\xi} \eta(\Lambda) &= -\dfrac{ 4 \left( 1 - \varepsilon \right) \left\|\xi\right\|^2 \xi}{ \varepsilon \xi^\top \zeta} + \dfrac{\left( 1 - \varepsilon \right) \left\|\xi\right\|^4 \zeta}{\varepsilon \left(\xi^{\top} \zeta\right)^2}, \\
		\nabla_{\zeta} \eta (\Lambda) &= \dfrac{\left(1 - \varepsilon \right) \left\| \xi \right\|^4 \xi}{\varepsilon \left( \xi^{\top} \zeta \right)^2}.
	\end{align*}
	This implies that $\nabla_\Lambda \eta$ is continuous in $\Lambda$. 
	Moreover, multiplying both sides of $g(\Lambda) < 1$ by $\xi^\top \zeta$ indicates $|\xi^\top \zeta| > (1-\varepsilon)\|\xi\|^2$, which leads to the bounds $\|\nabla_\xi \eta\| \le 4\|\xi\|/\varepsilon + \|\zeta\|/(\varepsilon(1-\varepsilon)) \le 2 / \sqrt{\varepsilon} + (\|\zeta_0\| + \sqrt{\varepsilon}/2) / (\varepsilon (1 - \varepsilon))$ and $\|\nabla_\zeta \eta\| \le \|\xi\|/(\varepsilon(1-\varepsilon)) \le 1/(2\sqrt{\varepsilon} (1 - \varepsilon) )$. 
	Then we have $\|\nabla_\Lambda \eta(\Lambda) \| \le \sqrt{\|\nabla_\xi \eta\|^2 + \|\nabla_\zeta \eta\|^2} \le \tilde{M}$ where $\tilde{M} = \sqrt{[2 / \sqrt{\varepsilon} + (\|\zeta_0\| + \sqrt{\varepsilon}/2) / (\varepsilon (1 - \varepsilon))]^2 + [1/(2\sqrt{\varepsilon} (1 - \varepsilon) )]^2} $, indicating that $\nabla_\Lambda \eta(\Lambda) $ is continuous and bounded in this case. 
	To sum up, for all $\Lambda \in U \setminus E$, we have $\nabla_\Lambda \eta(\Lambda)$ is continuous and $\|\nabla_\Lambda \eta(\Lambda)\| \le M$ with $M = \max \{\varepsilon^{-1/2} , \tilde{M}\}$. 
	
	The second step is to demonstrate the Lipschitz continuity of $\eta$ on $U$. 
	For $\Lambda_1, \Lambda_2 \in U$, the objective is to show that $|\eta(\Lambda_1)-\eta(\Lambda_2)| \le L_{\eta} \|\Lambda_1-\Lambda_2\|$ for some $L_{\eta} \ge 0$. 
	Define a line segment $\gamma(t) = (1-t)\Lambda_1 + t\Lambda_2$ where $t \in [0,1]$. 
	Let $S = \gamma^{-1}(E) = \{t \in [0,1] : \gamma(t) \in E\}$ be the preimage of $E$ under $\gamma$. 
	Then we analyze the Lebesgue measure of $S$. 
	Define $\xi_t = (1-t)\xi_1 + t\xi_2$ and $\zeta_t = (1-t)\zeta_1 + t\zeta_2$ for $t \in [0,1]$. 
	According to the definition of $S$, $t \in S \Leftrightarrow \xi_t^\top \zeta_t = 0 $ or $\xi_t^\top \zeta_t + (1-\varepsilon) \xi_t^\top \xi_t = 0$. 
	Then $S$ can be separated as a union of two sets as $S = S_1 \bigcup S_2$ where $S_1 = \{ t\in [0,1] : \xi_t^\top \zeta_t = 0 \}$ and $S_2 = \{ t\in [0,1] : \xi_t^\top \zeta_t + (1-\varepsilon) \xi_t^\top \xi_t = 0\}$. 
	Substituting $\xi_t$ and $\zeta_t$, 
	the sets $S_1$ and $S_2$ can be reformulated as $S_1 = \{ t\in [0,1] : At^2 + Bt + C = 0 \}$ and $S_2 = \{ t\in [0,1] : [A + (1-\varepsilon)P] t^2 + [B + (1-\varepsilon)Q ]t+[C+(1-\varepsilon)R] = 0\}$ where $A = (\xi_2 - \xi_1)^\top ( \zeta_2 - \zeta_1)$, $B  = \xi_1^\top ( \zeta_2 - \zeta_1 ) + ( \xi_2 - \xi_1 )^\top \zeta_1$, $C = \xi_1^\top \zeta_1$, and $P= \|\xi_2-\xi_1 \|^2$, $Q= 2 \xi_1^\top (\xi_2- \xi_1 )$, $R= \|\xi_1 \|^2$. 
	We first analyze the Lebesgue measure of $S_1$. 
	If $A, B, C$ are not all zero, the solutions to $At^2 + Bt + C = 0$ are finite (at most 2 solutions), indicating that $S_1$ has measure zero. 
	If $A=B=C=0$, $At^2 + Bt + C = 0$ holds trivially for all $t \in [0,1]$, which indicates that $S_1 = [0,1]$. 
	Therefore, $S_1$ either has measure zero or is equivalent to $[0,1]$. 
	Similarly, we can also obtain that $S_2$ either has measure zero or is equivalent to $[0,1]$. 
	Next, we show the Lipschitz of $\eta$ on $U$ by considering the following two cases. 
	
	Case 1: $S_1$ and $S_2$ both have measure zero. 
	If $S_1$ and $S_2$ both have measure zero, $S = S_1 \bigcup S_2$ also has measure zero. 
	Define $f(t) = \eta(\gamma(t))$. 
	Since $\gamma$ is continuously differentiable and $\eta$ is continuously differentiable on $ U\setminus E$, it follows that $f(t)$ is continuously differentiable on $[0,1] \setminus S$. 
	Therefore, $f^\prime(t)$ exists and is continuous on $[0,1] \setminus S$, and we have the bound $|f'(t)| = \left| \nabla_\Lambda \eta(\gamma(t))^\top (\Lambda_2 - \Lambda_1) \right| \leq M \|\Lambda_2 - \Lambda_1\|$ for $ t \in [0,1] \setminus S$. 
	Since $S$ has measure zero, $f^\prime(t)$ is integrable on $[0,1]$. 
	As a result, $ |\eta(\Lambda_1) - \eta(\Lambda_2)| = |f(0) - f(1)| \leq \int_0^1 |f'(t)|  dt \leq M \|\Lambda_2 - \Lambda_1\|$. 
	
	Case 2: Either $S_1$ or $S_2$ does not measure zero. 
	In this case, the above analysis shows that either $S_1 = [0,1]$ or $S_2 = [0,1]$ holds. 
	Then $[0,1] \subseteq S_1\bigcup S_2 = S \subseteq [0,1]$, which implies $S = [0,1]$. 
	By the definition of $S$, we obtain $\gamma(t) \in E$ for all $t\in [0,1]$. 
	Thus $\Lambda_1 = \gamma(0) \in E$ and $\Lambda_2 = \gamma(1) \in E$. 
	Then by the definition of $E$, we have $\eta_2 (\Lambda_1) = \eta_2(\Lambda_2) = 1$, which implies that $\eta(\Lambda_1) = \eta_1(\Lambda_1)$ and $\eta(\Lambda_2) = \eta_1(\Lambda_2)$. 
	Note that the gradient of $\eta_1$ to $\Lambda\in U$ can be bounded by $\|\nabla_{\Lambda} \eta_1\| \le 2\|\xi \|/\varepsilon \le 1/\sqrt{\varepsilon}$, which implies that $\eta_1$ is Lipschitz on $U$ with a Lipschitz constant $ L_{\eta_1} = 1/\sqrt{\varepsilon}$. 
	Therefore, $|\eta(\Lambda_1) - \eta(\Lambda_2)| = |\eta_1(\Lambda_1) - \eta_1(\Lambda_2)| \le L_{\eta_1} \|\Lambda_1 - \Lambda_2\|$. 
	Then by taking $L_{\eta} = \max\{M, L_{\eta_1}\}$, we obtain $ |\eta(\Lambda_1) - \eta(\Lambda_2)| \le L_{\eta} \|\Lambda_2 - \Lambda_1\|$ for all $\Lambda_1, \Lambda_2 \in U$. 
	Hence, $\eta$ is Lipschitz on $U$, and thus locally Lipschitz on $\{\Lambda \in \mathbb{R}^{2d}: \|\xi\| = 0\}$. 
	
	In summary, $\eta$ is locally Lipschitz in $\Lambda$ over $\mathbb{R}^{2d} = \{\Lambda \in \mathbb{R}^{2d}: \|\xi\|>0\} \cup \{\Lambda \in \mathbb{R}^{2d}: \|\xi\| = 0\}$. 
	Combining with the local Lipschitz continuity of $\tilde{v}_i^\mathrm{cv}$, we obtain that $v_i^\mathrm{cv}$ is locally Lipschitz.     
\end{proof}

\begin{proof}[Proof of Lemma~\ref{Lemma_Vi_Lipschitz}]
	According to Lemma~\ref{Lem_MS_Lipschitz} and \ref{Lem_CV_Lipschitz}, $v_i^\mathrm{ms}$ and $v_i^\mathrm{cv}$ is locally Lipschitz in $\mathbf{p}$. 
	In order to show that $v_i = \mathrm{sat}(v_i^\mathrm{ms} + v_i^\mathrm{cv})$ is locally Lipschitz in $\mathbf{p}$, we need only show that the function $\mathrm{sat}(z)$ is Lipschitz in $z \in \mathbb{R}^d$. 
	
	The function $\mathrm{sat}(z)$ can be reformulated as a projection onto an compact ball $C = \{ z \in \mathbb{R}^d : \|z\| \le v_\mathrm{max}\}$, i.e., $\mathrm{sat}(z) = \arg \min_{x \in C} \|x-z\|$. 
	Recall that $v_\mathrm{max}$ is the maximum speed of the robots. 
	Since $C$ is a nonempty closed convex set, we have $\|\mathrm{sat}(z_1) - \mathrm{sat}(z_2)\| \le \|z_1-z_2\|$ by Proposition 5.3 in \cite{Brezis2011Functional}. 
	Thus, $\mathrm{sat}(z)$ is globally Lipschitz in $z \in \mathbb{R}^d$. 
	
	In summary, the locally Lipschitz continuity of $\mathrm{sat}(\cdot)$, $v_i^\mathrm{ms}$ and $v_i^\mathrm{cv}$ have all been established, and hence we conclude that $v_i$ is locally Lipschitz in $\mathbf{p}$ over $\mathbb{R}^{dn}$.
\end{proof}

\section{Choice of $\beta$}
\label{Sec_Alg}

Algorithm \ref{Alg_DA} outlines a method for selecting the bandwidth parameter $\beta$ via a deterministic annealing approach \cite{Rose1998Deterministicannealing}. 
Given a set of sample points and a fixed number of robots, the algorithm determines a value of $\beta$ that ensures the swarm's equilibrium state—defined by the control law in \eqref{Equ_controllaw}—closely approximates the desired shape. 
The process initializes $\beta$ to a small value and incrementally increases it according to a cooling schedule. 
At each iteration, the formation quality under the current $\beta$ is evaluated by alternating between an E-step (updating robot-to-sample-point associations) and an M-step (updating robot positions). 
The procedure continues until convergence is achieved and inter-robot distances surpass a predefined threshold $d_\mathrm{min}$, yielding an optimal $\beta$.

\begin{algorithm}
	\caption{Deterministic Annealing for $\beta$}
	\label{Alg_DA}
	{
		\renewcommand{\algorithmiccomment}[1]{\textsl{// #1}}  
		\begin{algorithmic}[1]
			\State \textbf{Input:}
			\State \quad $\{q_k\}_{k=1}^m$: positions of $m$ sample points
			\State \quad $n$: number of robots
			\State \quad $d_{\text{min}}$: desired minimum inter-robot distance in the desired formation
			\State \textbf{Output:}
			\State \quad $\beta$: optimized annealing parameter 
			
			\State \textbf{Initialize parameters:}
			\State $\beta_{\text{initial}} \gets 0.01$ \Comment{Initial annealing parameter}
			\State $\beta_{\text{final}} \gets 150$ \Comment{Final annealing parameter} 
			\State $\alpha \gets 1.025$ \Comment{Cooling rate ($\alpha > 1$ for $\beta$ increase)}
			\State $\epsilon \gets 10^{-3}$  
			
			\State Initialize $\beta \gets \beta_{\text{initial}}$
			\State Initialize robots' positions $p_1 = \ldots = p_n = \frac{1}{m}\sum_{k=1}^m q_k$ 
			
			\While{$\beta < \beta_{\text{final}}$}
			\State \Comment{E-step: Update association probabilities} 
			\For{$k = 1$ to $m$} 
			\For{$i = 1$ to $n$} 
			\State $P(p_i | q_k) = \frac{\exp(-\beta \|q_k - p_i\|^2)}{\sum_{j=1}^{n} \exp(-\beta \|q_k - p_j\|^2)}$
			\EndFor
			\EndFor
			
			\State \Comment{M-step: Update robot positions} 
			\For{$i = 1$ to $n$}  
			\State $p_i^{\text{old}} = p_i$, $p_i = \frac{\sum_{k=1}^{m} q_k \cdot P(p_i | q_k)}{\sum_{k=1}^{m} P(p_i | q_k)}$
			\EndFor
			
			\If{$\max_i \|p_i - p_i^{\text{old}}\| < \epsilon$} 
			\State $d_{\text{actual}} \gets \min_{i \neq j} \|p_i - p_j\|$ 
			\If{$d_{\text{actual}} \geq d_{\text{min}}$}  
			\State \textbf{break}  
			\Else
			\State $\beta \gets \alpha \cdot \beta$ 
			\EndIf
			\EndIf
			\EndWhile
			
			\State \Return $\beta$
		\end{algorithmic}
	} 
\end{algorithm}

\ifCLASSOPTIONcaptionsoff
\newpage
\fi

\footnotesize
\bibliographystyle{IEEEtran}
\bibliography{IEEEabrv,reference}

\end{document}